\documentclass{article}

\usepackage{microtype}
\usepackage{booktabs} 
\usepackage{enumitem}

\usepackage{hyperref}

\usepackage[accepted]{icml2024}

\usepackage{amsmath}
\usepackage{amssymb}
\usepackage{mathtools}
\usepackage{amsthm}

\usepackage[capitalize,noabbrev]{cleveref}

\theoremstyle{definition}
\newtheorem{theorem}{Theorem}[section]
\newtheorem{proposition}[theorem]{Proposition}
\newtheorem{lemma}[theorem]{Lemma}
\newtheorem{corollary}[theorem]{Corollary}
\theoremstyle{definition}
\newtheorem{definition}[theorem]{Definition}

\theoremstyle{remark}
\newtheorem{remark}[theorem]{Remark}

\usepackage[textsize=tiny]{todonotes}
\usepackage{multicol}
\usepackage{multirow}

\usepackage{subcaption}
\usepackage{graphicx}

\icmltitlerunning{Compositional Curvature Bounds for Deep Neural Networks}

\begin{document}

\twocolumn[
\icmltitle{Compositional Curvature Bounds for Deep Neural Networks}



\icmlsetsymbol{equal}{*}

\begin{icmlauthorlist}
\icmlauthor{Taha Entesari}{equal,jhu}
\icmlauthor{Sina Sharifi}{equal,jhu}
\icmlauthor{Mahyar Fazlyab}{jhu}
\end{icmlauthorlist}

\icmlaffiliation{jhu}{Department of Electrical and Computer Engineering, Johns Hopkins University, Baltimore, United States of America}


\icmlcorrespondingauthor{Mahyar Fazlyab}{mahyarfazlyab@jhu.edu}

\icmlkeywords{Machine Learning, Curvature Estimation, Hessian Bounding, Second-order Robustness, Certified Robustness, Adversarial Robustness}

\vskip 0.3in
]



\printAffiliationsAndNotice{\icmlEqualContribution} 

\begin{abstract}
    %
    %
    %
    %
    A key challenge that threatens the widespread use of neural networks in safety-critical applications is their vulnerability to adversarial attacks. In this paper, we study the second-order behavior of continuously differentiable deep neural networks, focusing on robustness against adversarial perturbations. First, we provide a theoretical analysis of robustness and attack certificates for deep classifiers by leveraging local gradients and upper bounds on the second derivative (curvature constant). Next, we introduce a novel algorithm to analytically compute provable upper bounds on the second derivative of neural networks. This algorithm leverages the compositional structure of the model to propagate the curvature bound layer-by-layer, giving rise to a scalable and modular approach. The proposed bound can serve as a differentiable regularizer to control the curvature of neural networks during training, thereby enhancing robustness. Finally, we demonstrate the efficacy of our method on classification tasks using the MNIST and CIFAR-10 datasets.
\end{abstract}

\section{Introduction}
Neural networks are infamously prone to adversarially designed perturbations \cite{szegedy2013intriguing}.
To address this vulnerability, many methods have been proposed to quantify and improve the robustness of these models against adversarial attacks, such as adversarial training \cite{zhang2019theoretically, madry2017towards}, regularization \cite{leino2021globally, tsuzuku2018lipschitz}, randomized smoothing \cite{cohen2019certified, kumar2021policy}, and many others.
One measure of robustness is the Lipschitz constant defined as the smallest $L_f \geq 0$ such that
\[
\|f(x)-f(y)\| \leq L_f \|x-y\| \quad \forall x,y.
\]
This constant quantifies the sensitivity of the model $f$ to input perturbations, motivating the need to estimate $L_f$ and control it through architecture or the training process.

For continuously differentiable functions, $L_f$ is a tight upper bound on the \emph{first} derivative ($\|Df(x)\| \leq L_f \ \forall x$). However, one can go one step further and leverage the smoothness of the first derivative, i.e., bounds on the \emph{second} derivative to obtain a more refined measure of the function's sensitivity. Indeed, the merit of second-order information in characterizing and enhancing robustness has been established, e.g., \cite{singla2020second}.


\textbf{Our Contributions:} In this work, we seek to characterize the adversarial robustness of continuously differentiable neural network classifiers through the Lipschitz constant of their first derivative defined as 
\[
\|Df(x)-Df(y)\| \leq L_{Df} \|x-y\| \quad \forall x,y
\]
If $f$ is twice differentiable,  this constant is a tight upper bound on the second derivative, $\|D^2f(x)\| \leq L_{Df} \  \forall x$. With a slight abuse of the formal definition, we denote $L_{Df}$ as the ``curvature'' constant. Our contributions are as follows.




\begin{itemize}[leftmargin=*]
    \item We provide a theoretical analysis of the interplay between adversarial robustness and smoothness. Specifically, for classification tasks, we derive lower bounds on the margin of correctly classified data points using the first derivative (the Jacobian) and its Lipschitz constant (the curvature constant). We then show that these curvature-based certificates provably improve upon Lipschitz-based certificates, provided the curvature is sufficiently small. 
    
    \item We propose a novel algorithm to derive analytical upper bounds on the curvature constant of neural networks. This algorithm leverages the compositional structure of the model to compute the bound in a scalable and modular fashion, improving upon previous works that only consider scalar-valued networks with affine-then-activation architectures. The derived bound is differentiable and can be used as a regularizer for training low-curvature neural networks.
    \item We introduce a relaxed notion of smoothness, the \textit{Anchored Lipschitz} constant, which significantly reduces conservatism in terms of robustness certification. Succinctly, this definition fixes one of the two points involved in the definition of Lipschitz continuity to the point of interest.
    \item We also present empirical results demonstrating the performance of our method compared to previous works on calculating curvature bounds, and we examine the impact of low curvatures on the robustness of deep classifiers. 
\end{itemize}

To the best of our knowledge, this paper is the first to develop a method for obtaining provable bounds on the second derivative of \textit{general} sequential neural networks. While we consider adversarial robustness as an application domain, the proposed method is also of independent interest for other applications requiring differentiable bounds on the second derivative of neural networks, such as learning-based control for safety-critical applications \cite{robey2020learning}.

\subsection{Related Work}
    With respect to the large body of work in this field, here, we focus on the works that are more relevant to our setup.
    \paragraph{Adversarial Robustness:}
    The robustness of deep models against adversarial perturbations has been a topic of interest in recent years \cite{singla2021skew, singla2022improved, xu2023exploring, zou2023universal}. 
    \cite{huang2021training, fazlyab2023certified} use the Lipschitz constant of the network during the training procedure to induce robustness by bounding the worst-case logits. 
    To achieve robustness, instead of penalizing or constraining the Lipschitz constant during training, some methods directly construct 1-Lipschitz networks.
    The use of Lipschitz bounded networks has been encouraged by many recent works \cite{bethune2022pay} as they provide desirable properties such as robustness and improved generalization.
    AOL \cite{prach2022almost} provides a rescaling of the layer weights that makes each linear layer 1-Lipschitz.
    To obtain Lipschitz bounded networks, many works have utilized LipSDP \cite{fazlyab2019efficient} to parameterize 1-Lipschitz layers.
    SLL \cite{araujo2023unified} proposes 1-Lipschitz residual layers by satisfying LipSDP, and \cite{fazlyab2023certified} generalizes SLL by proposing a $\sqrt{\rho}$-Lipschitz layer.
    Most recently, \cite{wang2023direct} satisfies the LipSDP condition using Caley Transforms and proposes a non-residual 1-Lipschitz layer.

    Other works look beyond the network's first-order properties and control the network's curvature \cite{moosavi2019robustness, singla2021low}. \cite{srinivas2022efficient} proposes using centered-soft plus activations and Lipschitz-bounded batch normalizations to cap the curvature and empirically improve robustness.
    
    %
    \paragraph{Lipschitz Constant Calculation:}
    In recent years, there has been a focus on finding accurate bounds on the Lipschitz constant of neural networks. Here we only discuss the ones that can handle continuously-differentiable networks. One of the early works, \cite{szegedy2013intriguing}, provided a bound on the Lipschitz constant using the norm of each layer, which is known to be a loose bound. \cite{fazlyab2019efficient} formulated the problem of finding the Lipschitz constant as a semidefinite program (SDP), providing accurate bounds but at the expense of limited scalability. Later, \cite{hashemi2021certifying} introduced a local version of LipSDP. Most recently, \cite{fazlyab2023certified} proposed LipLT, an analytic method for bounding the Lipschitz constant through loop transformation, a control-theoretic concept. In this work, we also leverage LipLT to derive upper bounds on the curvature constant.
    
    Most relevant to our setup, \cite{singla2020second} develops a method to bound the curvature constant of scalar-valued neural networks in the \(\ell_2\) norm and introduces a numerical optimization scheme to provide curvature-based certificates. In contrast, our method bounds the curvature of arbitrary function compositions, in particular vector-valued feedforward neural networks, in any \(\ell_p\) norm, and provides analytical curvature-based certificates.

    %
    

\subsection{Preliminaries and Notation}
We denote the $n$-dimensional real numbers as $\mathbb{R}^n$. 
For a vector $x \in \mathbb{R}^n$, $x_i$ is its $i$-th element. 
For a matrix $W \in \mathbb{R}^{n \times m}$, $W_{i, :} 
 \in \mathbb{R}^{1 \times m}, W_{:, j} \in \mathbb{R}^{n}, W_{i, j} \in \mathbb{R}$ are the $i$-th row, $j$-th column, and the $j$-th element of $W_{i, :}$, respectively. For a vector $x$, $\mathrm{diag}(x)$ is the diagonal matrix with $\mathrm{diag}(x)_{ii} = x_i$ and zero otherwise. For an integer $n$ let $[n] = \{1, \cdots, n\}$.
Moreover, the operator norm of a matrix $A$ is denoted as
    $
    \|A\|_{p\rightarrow q} = \sup_{\|x\|_p \leq 1} \|Ax\|_q
    $.
For a real-valued $p \geq 1$, we denote its H$\mathrm{\ddot{o}}$lder conjugate with ${p^*}$, i.e., $\frac{1}{p} + \frac{1}{{p^*}} = 1$. For any vector $x \in \mathbb{R}^n$ and norm $\|\cdot\|_p$, we have $\|x\|_{p^*} = \sup_{\|y\|_p \leq 1} x^\top y$. 

A function $f: \mathbb{R}^n \rightarrow \mathbb{R}^m$ is Lipschitz continuous on $\mathcal{C} \subseteq \mathbb{R}^n$ if there exists a non-negative constant $L_f^{p, q}$ such that 
$\| f(x) - f(y)\|_q \leq L_f^{p, q} \| x - y \|_p \ \forall x,y \in \mathcal{C}$.  The smallest such $L_f^{p, q}$ is \emph{the} Lipschitz constant, in the corresponding norms, which is given by
\[
L_f^{p, q} = \sup_{x,y \in \mathcal{C}, x \neq y} \frac{\| f(x) - f(y)\|_q }{ \| x - y \|_p}.
\]
For brevity, we denote $L_f^{p, p}$ as $L_f^p$.
In this work, we define a new notion of Lipschitz continuity at a neighborhood of a point. 
\begin{definition}[Anchored Lipschitz constant]
    For a function $f$, the \emph{anchored} Lipschitz constant at a point $x \in \mathcal{C}$ is defined as
    \[
    L_f^{p, q}(x) = \sup_{y \in \mathcal{C}, x \neq y} \frac{\| f(x) - f(y)\|_q }{ \| x - y \|_p}.
    \]
\end{definition}
At any point $x$, this constant 
is a lower bound on the Lipschitz constant as one can confirm $L_{f}^{p, q} = \sup_{x \in \mathcal{C}} L_{f}^{p, q}(x)$.
\Cref{fig:anch} demonstrates this concept further for the specific case of the $\tanh$ function. 

In the following lemma, we establish the relation between the anchored Lipschitz constant and the norm of the derivative.
\begin{lemma}\label{lemma:lipschitzUpperBoundOnNormJacobian}
    Consider a differentiable function $f: \mathbb{R}^n \to \mathbb{R}^m$ and let $L_f^{p}(x)$ be a corresponding anchored Lipschitz constant. We have
    \[
    \|Df(x)\|_p \leq L_f^p(x).
    \]    
\end{lemma}
See \Cref{par:radiiComp} for the proof.

Given bounded numbers $\alpha \leq \beta$, a function $\phi: \mathbb{R} \rightarrow \mathbb{R}$ is slope restricted in $[\alpha, \beta]$ if $$\alpha \leq \frac{\phi(x) - \phi(y)}{x - y} \leq \beta, \quad \forall x, y.$$
The Lipschitz constant of $\phi$ is then $L_\phi = \max(|\alpha|,|\beta|)$.
For simplicity, and based on commonly-used differentiable activation functions such as sigmoid and tanh, we assume that $\phi$ is monotone, i.e., $\alpha \geq 0$, implying that $L_\phi = \beta$.
%


\section{Curvature-based Robustness Analysis}\label{sec:curvatureBasedAnalysis}

Consider a continuously differentiable function $f: \mathbb{R}^{n} \rightarrow \mathbb{R}^{m}$ parameterized by a neural network. In this work, our goal is to derive provable upper bounds on the Lipschitz constant of the Jacobian $Df: \mathbb{R}^{n} \rightarrow \mathbb{R}^{m \times n}$, defined as the smallest constant $L_{Df}^{p, q}$ such that
\begin{align}
    &\|Df(x_1) \!-\! Df(x_2) \|_q \leq L_{Df}^{p, q} \| x_1 \!-\! x_2 \|_p, \quad \forall x_1,x_2. \notag
\end{align}
Furthermore, we can extend this to the anchored Lipschitz constant of the Jacobian,  $L_{Df}^{p, q}(x)$, at a given point $x$, as
\begin{align}
    &\|Df(x + \delta) \!-\! Df(x) \|_q \leq L_{Df}^{p, q}(x) \| \delta \|_p, \quad \forall \delta. \notag
\end{align}

While providing provable upper bounds on these constants can be instrumental in various applications, in this work, we primarily focus on the adversarial robustness of deep classifiers. We develop our methods and certificates based on the Lipschitz continuity of the classifier and its Jacobian. We elaborate more on this in the following subsections.

\subsection{Robustness Certificates for Deep Classifiers}
Consider a classifier $C(x):=\arg\max_{1 \leq i \leq n_K} f_i(x)$ with $n_{K}$ classes, where $f: \mathbb{R}^{n_0} \rightarrow \mathbb{R}^{n_K}$ is a neural network that parameterizes the vector of logits. 
For a given input $x$ with correct label $y \in [n_K]$, the condition for correct classification is 
\[
f_{iy}(x) := f_i(x) - f_y(x)<0, \quad \forall i\neq y.
\]
Assuming that $x$ is correctly classified as $y$, the distance of $x$ to the closest decision boundary measures the classifier's local robustness against additive perturbations. We can compute this distance by solving the following optimization problem,
%
%
\begin{align}\label{eq:defenseRadius}
    \begin{split}
        \varepsilon^*(x) = \max &\quad \varepsilon \\
        \text{s.t.} &\quad \sup_{\| \delta \|_p \leq \varepsilon} f_{iy}(x + \delta) \leq 0, \quad \forall i \neq y.
    \end{split}
\end{align}

For ReLU networks and $p \in \{1,\infty\}$, this optimization problem can be encoded as a Mixed-Integer Linear program (MILP) by exploiting the piece-wise nature of the activation functions \cite{dutta2018output, fischetti2018deep, tjeng2017evaluating}. While these MILPs can be solved globally, they suffer from poor scalability. For neural networks with differentiable activation functions, even this mixed-integer structure is absent, making the exact computation of distances effectively intractable. Therefore, we must resort to finding lower bounds on the certified radius to gain tractability.

\subsubsection{Lipschitz-based certificates} Suppose the $f_{iy}$'s are Lipschitz continuous. We can then write
\begin{equation}\label{eq: anchored Lip for logit difference}
    f_{iy}(x + \delta) \leq f_{iy}(x) + L_{f_{iy}}^p(x) \|\delta\|_p.
\end{equation}
where $L_{f_{iy}}^p(x)>0$ is the \emph{anchored} Lipschitz constant of $f_{iy}$.
By substituting \eqref{eq: anchored Lip for logit difference} in the constraints of \eqref{eq:defenseRadius}, we obtain the following optimization problem to compute a zeroth-order (gradient-free) lower bound,
\begin{align*}
\begin{split}
        \underline{\varepsilon}_0^*(x) = &\max \ \varepsilon \\
        &\text{s.t.} \sup_{\| \delta \|_p \leq \varepsilon} f_{iy}(x) + L_{f_{iy}}^p(x) \|\delta\|_p \leq 0,  \forall i \neq y.
\end{split}
\end{align*}
We note that due to constraint tightening, we have 
$\underline{\varepsilon}_0^*(x)< {\varepsilon}^*(x)$. Using similar arguments as in \cite{fazlyab2023certified}, $\underline{\varepsilon}_0^*(x)$ has the closed-form expression 
\begin{align}\label{eq:neuripsLipschitzCertificate}
   \underline{\varepsilon}_0^*(x) = \min_{i \neq y}  \frac{-f_{iy}(x)}{L_{f_{iy}}^p(x)}.
\end{align}

\subsubsection{Curvature-based certificates} When the model is continuously differentiable, we can exploit its curvature to improve the certificate in \eqref{eq:neuripsLipschitzCertificate}. Specifically, suppose the logit difference $f_{iy}$ is continuously differentiable with Lipschitz gradients and let $L_{\nabla f_{iy}}^{p, p^*}(x)$ be an anchored Lipschitz constant of $\nabla f_{iy}$ at $x$. Then, we can compute an upper bound on $f_{iy}(x+\delta)$ as follows,
\begin{equation}\label{eq:descentLemma}
    f_{iy}(x + \delta) \leq 
     \underbrace{f_{iy}(x) + \nabla f_{iy}(x)^\top \delta + \frac{L_{\nabla f_{iy}}^{p, p^*}(x)}{2}\|\delta\|_p^2}_{\overline{f_{iy}}(x, \delta; L_{\nabla f_{iy}}^{p, {p^*}}(x))}.
\end{equation}
See \Cref{par:proopOfSecondOrderRadius} for a derivation of this inequality. In contrast to the zeroth-order bound, this upper bound uses the local first derivative, $\nabla f_{iy}(x)$,  as well as bounds on its (anchored) Lipschitz constant, $L_{\nabla f_{iy}}^{p, p^*}(x)$,  to obtain a locally more accurate approximation of $f_{iy}(x+\delta)$. By substituting the upper bound \eqref{eq:descentLemma}  in \eqref{eq:defenseRadius}, we obtain a first-order (gradient-informed) lower bound on $\varepsilon^*(x)$,
\begin{align}\label{eq:defenseRadiusRelaxed1}
        \begin{split}
            \underline{\varepsilon}_1^*(x) = &\max \ \varepsilon \\
        &\text{s.t.} \sup_{\| \delta \|_p \leq \varepsilon} \overline{f_{iy}}(x, \delta; L_{\nabla f_{iy}}^{p, {p^*}}(x)) \leq 0, \quad \forall i \neq y
        \end{split}
\end{align}

\begin{figure}[t]
    \centering
    \includegraphics[width=0.85\columnwidth]{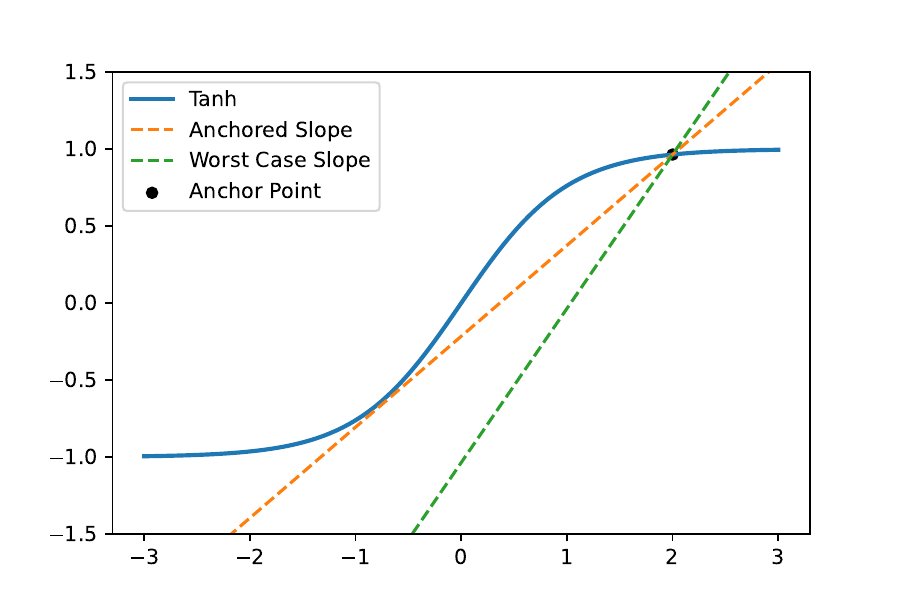}
    \caption{Depiction of anchored Lipschitz constants for $f(x) = \tanh(x)$. The anchored Lipschitz constant at $x=2$ is less than $0.582$, whereas the global Lipschitz constant is $1$.}
    \label{fig:anch}
\end{figure}
As we summarize below, we can compute this lower bound in closed form, provided that we can compute $L_{\nabla f_{iy}}^{p, {p^*}}(x)$. 
\begin{proposition}[Curvature-based certified radius]\label{prop:secondOrderCertifiedRadius} 
    Suppose $x$ is classified correctly, i.e., $f_{iy}(x)<0 \ \forall i \neq y$. The optimization problem \eqref{eq:defenseRadiusRelaxed1} has the closed-form solution $\underline{\varepsilon}_1^*(x)$ given by
    \begin{align}\label{eq:icmlLipschitzCertificate}
        \min_{i\neq y} \! \frac{-\| \nabla f_{iy}(x) \|_{p^*} \! + \! (\| \nabla f_{iy}(x) \|_{p^*}^2 \! - \! 2 L_{\nabla f_{iy}}^{p, {p^*}}(x) f_{iy}(x))^{\frac{1}{2}}}{L_{\nabla f_{iy}}^{p,{p^*}}(x)}.
    \end{align}
\end{proposition}
See \Cref{par:proofOfCertRad} for the proof of this proposition.


In the following proposition, we show that if the curvature of the model is sufficiently small, we can certify a larger radius than Lipschitz-based certificates.

\begin{proposition}\label{prop:compare}
    Suppose $x$ is classified correctly, i.e., $f_{iy}(x)<0 \ \forall i \neq y$. Fix a $p \geq 1$, and define the zeroth-order $\underline{\varepsilon}_0^*(x)$ and first-order $\underline{\varepsilon}_1^*(x)$ certified radii as in \eqref{eq:neuripsLipschitzCertificate} and \eqref{eq:icmlLipschitzCertificate}. If the following condition holds,
    \begin{align}
         L_{\nabla f_{iy}}^{p, p^*}(x) \leq \frac{-2(\| \nabla f_{iy}(x) \|_{p^*} \underline{\varepsilon}_0^*(x) + f_{iy}(x))}{\underline{\varepsilon}_0^*(x)^2}, \quad i \neq y. \notag
    \end{align}
    Then $\underline{\varepsilon}_1^*(x) \geq \underline{\varepsilon}_0^*(x)$. 
\end{proposition}
See \Cref{par:radiiComp} for the proof.

\subsection{Attack Certificates for Deep Classifiers}
Considering the same setup as before, we now aim to obtain the smallest perturbation by which a correctly classified data point can provably be misclassified.
This computation can be formulated as the following optimization problem,
\begin{align}\label{eq:attackRadius}
    \begin{split}
        {\varepsilon^{'}}^*(x) = \min &\quad \varepsilon \\
        \text{s.t.} &\quad \min_{i \neq y}\inf_{\| \delta \|_p \leq \varepsilon} f_{yi}(x + \delta) < 0. 
    \end{split}
\end{align}

First, we note that problems \eqref{eq:defenseRadius} and \eqref{eq:attackRadius} are equivalent.
\begin{proposition}\label{prop:equivalenceOfAttackAndDefence}
    Suppose $f$ correctly classifies the data point $x$ as $y$, i.e., $f_{iy}(x)<0$ for $i \neq y$. Then the optimal value of problems \eqref{eq:defenseRadius} and \eqref{eq:attackRadius} are equal, i.e., $\varepsilon^*(x) = {\varepsilon^{'}}^*(x)$.
\end{proposition}

Using the curvature-based upper bound, one can tighten the constraints of the problem and achieve a first-order (gradient-informed) attack certificate as follows,
\begin{align}\label{eq:attackRadiusRelaxed1}
        \begin{split}
            \overline{\varepsilon}_1^*(x) = &\min \varepsilon \\
        &\text{s.t.} \min_{i \neq y} \inf_{\| \delta \|_p \leq \varepsilon} \overline{f_{yi}}(x, \delta; L_{\nabla f_{yi}}^{p, {p^*}}(x)) < 0, 
        \end{split}
\end{align}

We analytically acquire the optimal value of this problem in the following proposition.
\begin{proposition}[Curvature-based attack certificate]\label{prop:AttackCert}
    Suppose $x$ is classified correctly, i.e., $f_{iy}(x)<0 \ \forall i \neq y$. Let $\mathcal{I} = \{ i | i \neq y, 2 L_{\nabla f_{yi}}^{p, {p^*}}(x) f_{yi}(x) \leq \| \nabla f_{yi}(x) \|_{p^*}^2   \}$. 
    Assuming that $\mathcal{I}$ is non-empty, the optimization problem \eqref{eq:attackRadiusRelaxed1} has the closed-form solution $\overline{\varepsilon}_1^*(x)$ given by 
    \begin{align}\label{eq:attackCert}
        \min_{i \in \mathcal{I}} \! \frac{\| \nabla f_{yi}(x) \|_{p^*} \! - \! (\| \nabla f_{yi}(x) \|_{p^*}^2 \! - \! 2 L_{\nabla f_{yi}}^{p, {p^*}}(x) f_{yi}(x))^{\frac{1}{2}}}{L_{\nabla f_{yi}}^{p, {p^*}}(x)}.
    \end{align}
    Given $i^* \in \mathcal{I}$ minimizing \eqref{eq:attackCert}, the perturbation realizing the attack certificate is obtained through solving $\sup_{\|\delta\|_p \leq \varepsilon} \nabla f_{i^*y}(x)^\top \delta$.
\end{proposition}
\begin{figure}[t]
    \centering
    \includegraphics[width=0.8\columnwidth]{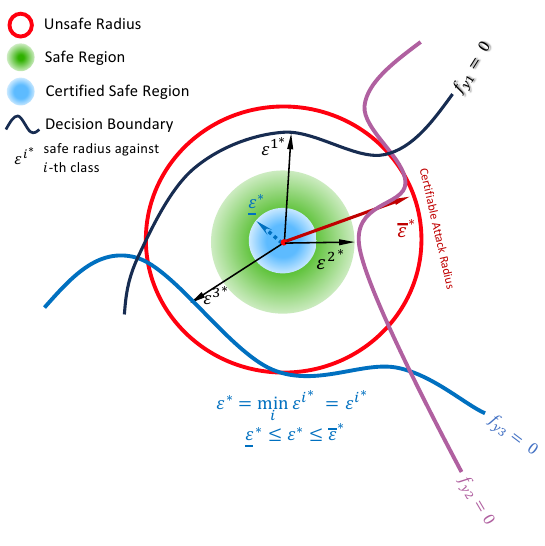}
    \caption{Certified ($\underline{\varepsilon}^* $) and attack ($\overline{\varepsilon}^*$) radii estimates. The green tangent circle denotes the certified radius $\varepsilon^*$.}
    \label{fig:certifiedRadiusAttack}
\end{figure}
We note that while problem  \eqref{eq:attackRadius} is always feasible (for a non-trivial classifier), problem \eqref{eq:attackRadiusRelaxed1} can be infeasible due to the tightening of the constraints, which is equivalent to the set $\mathcal{I}$ being empty. \Cref{fig:certifiedRadiusAttack} illustrates an example scenario for $\underline{\varepsilon}^*$ and $\overline{\varepsilon}^*$.


The derivation of $\underline{\varepsilon}^*(x)$ and $\overline{\varepsilon}^*(x)$ is significant in two ways. First, we established an analytical solution to the curvature-based certified radius in \Cref{prop:secondOrderCertifiedRadius}. Although curvature-based certificates have been studied in \cite{singla2020second}, their method involves an iterative algorithm to solve an optimization problem numerically, whereas our method yields closed-form solutions. 
Second, we introduced curvature-based attack certificates, a novel method for narrowing the \textit{certification gap} of classifiers.
The \textit{certification gap} is the (empirical) probability quantifying correctly classified points that lack the desired level of certified defense radii and lack attack certificates at a given perturbation budget $\epsilon$, i.e., $\mathbb{P}_{(x, y) \sim D}\{\underline{\varepsilon}(x) \leq \epsilon \leq \overline{\varepsilon}(x)\}$. Refer to \Cref{fig:attack_example} for an illustration.

\section{Efficient Estimation of Curvature Bounds}


    

Having established the importance of curvature in providing robustness and attack certificates, in this section, we propose our method to derive upper bounds on the Lipschitz constant of the Jacobian (the curvature constant) of general sequential models.
We then curate our algorithm for residual neural networks and explore various Lipschitz estimation techniques to calculate the curvature constant.

\subsection{Curvature Bounds for Composed Functions}
Let $h = f \circ g$ be the composition of two continuously differentiable functions $g: \mathbb{R}^{n_1} \to \mathbb{R}^{n_2}$  and $f: \mathbb{R}^{n_2} \to \mathbb{R}^{n_3}$ with Lipschitz Jacobians. The Lipschitz constant of the Jacobian $Dh \in \mathbb{R}^{n_3 \times n_1}$ of $h$ is an upper bound on the \emph{second derivative} of $h$, assuming it exists,  which is a third-order tensor \cite{srinivas2022efficient} and difficult to characterize.
Our goal is to compute the Lipschitz constant of $Dh$ directly without resorting to any tensor calculus. 

Using the chain rule, the Jacobian of $h$ can be written as 
$$Dh(x) = Df(g(x)) Dg(x).$$

The following theorem establishes a relation between the Lipschitz constant of $Dh$ and the Lipschitz constants of $f, g, Df$, and $Dg$. 
\begin{theorem} [Compositional curvature estimation] \label{prop:recursiveJacobianLipschitz}
        Given functions $f, g$, and $h$ as described above, the following inequality holds,
        \begin{align}\label{eq:curvatureRecursion}
            L_{Dh}^{p, p^*} \leq L^{p,p^*}_{Dg} L^{p^*}_f  + L^{p, p^*}_{Df} L_g^{p} L_g^{p^*},
        \end{align}
        where $L^{p, q}_s$ denotes the Lipschitz constant of the function $s(\cdot)$.
\end{theorem}

\Cref{prop:recursiveJacobianLipschitz} provides a basis to recursively calculate a Lipschitz constant for the Jacobian of the composition of multiple functions. In the following, we adapt this result to anchored Lipschitz constants. 
\begin{theorem}[Anchored compositional curvature estimation]\label{prop:anchoredIterativeJacobian}
    Consider functions $f, g,$  and $h$ as in \Cref{prop:recursiveJacobianLipschitz}. 
    The following inequality holds for the anchored Lipschitz constant of the Jacobian of $h$ at $x$
    \begin{equation*}
        L^{p, p^*}_{Dh}(x) \leq L^{p^*}_{f} L^{p,{p^*}}_{Dg}(x) + \|Dg(x)\|_{p^*} L^{p, {p^*}}_{Df}(g(x)) L^{p}_g(x),
    \end{equation*}
    where $L^{p, q}_s(x)$ denotes the anchored Lipschitz constant of the function $s(\cdot)$ at  $x$.
\end{theorem}

The structure of \Cref{prop:recursiveJacobianLipschitz} (and similarly \Cref{prop:anchoredIterativeJacobian}) is of particular interest for sequential neural networks that are the composition of individual layers. In the following section, we will instantiate our framework for such models.
\begin{remark}
    We note that in Theorems \ref{prop:recursiveJacobianLipschitz} and \ref{prop:anchoredIterativeJacobian}, the dual norm $p^*$ is chosen to tailor the bounds specifically for \Cref{prop:secondOrderCertifiedRadius} and \Cref{prop:AttackCert}. 
    In general, the same statements hold if we replace $p^*$ with a general $q \geq 1$.
    See \Cref{par:proofCompCurv} for more details.
\end{remark}

\subsection{Curvature Bounds for Sequential Neural Networks}

Consider a sequential residual neural network
\begin{align}\label{eq:sequential model}
    x^{k+1} = h^k(x^k) = H^k x^k + G^k \Phi(W^k x^k),
\end{align}
where $k=0, 1, \cdots K-1$ and $W^k \in \mathbb{R}^{n^\prime_{k} \times n_{k}}, G^k \in \mathbb{R}^{n_{k + 1} \times n^\prime_{k}}, $ and $H^k \in \mathbb{R}^{n_{k + 1} \times n_{k}}$ are general matrices. 
For $x \in \mathbb{R}^n$, $\Phi(x) = [\phi(x_1), \cdots, \phi(x_n)]^\top$,  where $\phi$ is a differentiable monotone activation function slope-restricted in  $[\alpha, \beta]$ ($0 \leq \alpha \leq \beta < \infty$)  with its derivative slope-restricted in $[\alpha', \beta']$  ($-\infty < \alpha' \leq \beta' < \infty$). By setting $H^k=0$ and $G^k=I$, we obtain the standard feedforward architecture.

Leveraging \Cref{prop:recursiveJacobianLipschitz}, we propose a recursive algorithm to compute an upper bound on the Jacobian of the end-to-end map $x^0 \mapsto x^{K}$.
We establish this algorithm in \Cref{prop:networkHessianCalculationAlgorithm}.

\begin{corollary}
\label{prop:networkHessianCalculationAlgorithm}
    Let $\overline{L}^{p, p^*}_{D_k}$, $k=0,\cdots, K - 1$ be defined recursively as
    \begin{align} \label{eq: curvature bound for sequential networks}
        \overline{L}^{p, p^*}_{D_{k+1}} = \overline{L}^{p, p^*}_{Dh^k} \overline{L}^p_{k} \overline{L}^{p^*}_{k} + \overline{L}^{p^*}_{h^k} \overline{L}^{p, p^*}_{D_k}, 
    \end{align}
    with $\overline{L}^{p, p^*}_{D_0} = 0$, $\overline{L}^p_0 = \overline{L}^{p^*}_0=1$, where $\overline{L}^p_k,\overline{L}^{p^*}_k$ are Lipschitz constants for the map $x^0 \mapsto x^k$, and $\overline{L}^{p, p^*}_{Dh^k}$ is a Lipschitz constant for the Jacobian of $h^k$. Then $\overline{L}^{p, p^*}_{D_k}$ is a Lipschitz constant for the Jacobian of the map $x^0 \mapsto x^k$.
\end{corollary}
%
 Given upper bounds on the Lipschitz constants as $\overline{L}^p_k$, $\overline{L}^{p^*}_k$, 
$\overline{L}^{p^*}_{h^k}$, and $\overline{L}^{p, p^*}_{Dh^k}$, \Cref{prop:networkHessianCalculationAlgorithm} presents an algorithm to calculate an upper bound on the curvature constant of residual neural networks in a layer-by-layer fashion.

Next, we will compute the individual constants appearing in \eqref{eq: curvature bound for sequential networks}.

\subsubsection{Computation of $\overline{L}^p_{h^k}$} 
$L^p_{h^k}$ is the Lipschitz constant of the $k$-th layer $h^k$. Starting from \eqref{eq:sequential model}, an analytical upper  bound on this constant is
\begin{equation}\label{eq:naiveLipschitzOfResidual}
    \overline{L}_{h^k}^{p, \text{naive}} = \|H^k\|_p + \beta \|G^k\|_p \| W^k\|_p.
\end{equation}
This bound is relatively crude as it does not exploit the monotonicity of the activations, i.e., \eqref{eq:naiveLipschitzOfResidual} is agnostic to the value of $\alpha $.
As proposed in \cite{fazlyab2023certified}, this bound can be improved by applying a loop transformation on the activation layer $\Phi$. Specifically,  we can rewrite $h^k$ as
\begin{align} \label{loop transformed layer}
    h^k(x^k) = \hat{H}^k x^k + G^k \Psi(W^k x^k),
\end{align}
where $\hat{H}^k = H^k\!+\!\frac{\alpha+\beta}{2}G^kW^k$ and $\psi(z)=\phi(z)-(\alpha+\beta)z/2$ is the loop transformed activation layer. As a result of this transformation, $\Psi$ is now slope-restricted in $\frac{\beta-\alpha}{2}[-1,1]$, implying that $\psi$ is $(\frac{\beta-\alpha}{2})$-Lipschitz. An upper bound on the Lipschitz constant of $h^k$, reformulated as in \eqref{loop transformed layer}, is then
\begin{align}\label{eq:LipLT single layer}
    \overline{L}_{h^k}^{p, \text{LT}} = \|H^k\!+\!\frac{\alpha\!+\!\beta}{2}G^kW^k\|_p +\frac{\beta-\alpha}{2}\|G^k\|_p \| W^k\|_p.
\end{align}
As shown in \cite{fazlyab2023certified}, this bound, now informed by the monotonicity constant $\alpha$, is provably better than \eqref{eq:naiveLipschitzOfResidual}. This can be proved by applying the triangle inequality on the first term.  

\subsubsection{Computation of $\overline{L}^p_k$.}
$L^p_k$ is the Lipschitz constant of the map $x^0 \mapsto x^k$ defined by the composed function $(h^{k-1} \circ \cdots \circ h^0)(x^0)$. A naive bound on $L^p_k$ is the product of the Lipschitz constant of individual layers, i.e., $\overline{L}_k^{p, \text{naive}} = \prod_{i=0}^{k - 1} L^p_{h^i}$, where we can upper bound each $L^p_{h^i}$ from \eqref{eq:LipLT single layer}. However, this bound can grow quickly as the depth increases. To mitigate the adverse effect of depth, we exploit the idea of LipLT. Specifically, we can \emph{unroll} \eqref{eq:sequential model} after applying loop transformation to all activation layers, resulting in
\begin{align} 
    \begin{split}
        x^{k + 1} \!&=\! \hat{H}^{k} \cdots \hat{H}^0 x^0 \!+\! \sum_{j = 0}^{k} \hat{H}^{k} \cdots \hat{H}^{j + 1}G^j\Psi(W^j x^j). \notag
    \end{split}
\end{align}
This representation enables us to obtain all the constants $\overline{L}_1^{p, \text{LT}},\cdots,\overline{L}_K^{p, \text{LT}}$ recursively as follows,
\begin{align}\label{eq:lipLtRecursive}
    \begin{split}
        \overline{L}_{k+1}^{p, \text{LT}} \!&=\! \|\hat{H}^{k} \!\cdots\! \hat{H}^0\|_p \\ &\!+\! 
    \tfrac{\beta\!-\!\alpha}{2}\sum_{j=0}^{k} \|\hat{H}^{k} \cdots \hat{H}^{j + 1}G^j\|_p \|W^j\|_p\overline{L}_j^{p,\text{LT}}.
    \end{split}
\end{align}
As shown in \cite{fazlyab2023certified}, this bound provably improves the naive bound obtained by the product of Lipschitz constants of individual layers, i.e., $ \overline{L}_{k}^{p, \text{LT}} \leq \prod_{i=0}^{k - 1} \overline{L}^{p, \text{naive}}_{h^i}$. 

\subsubsection{Computation of $\overline{L}^{p, p^*}_{Dh^k}$.}\label{subsec:lipSdpJacLip}

Consider the $k$-th residual block $h^k$ in \eqref{eq:sequential model}.  The Jacobian of this block is given as
    \begin{equation}\label{eq:JacobianFormula}
        Dh^k(x^k) = H^k + G^k \: \mathrm{diag}(\Phi'(W^kx^k)) W^k.
    \end{equation}
    The following proposition provides an upper bound on the Lipschitz constant of this differential operator.
    \begin{proposition}\label{prop:naiveJacLip}
        The Jacobian $Dh^k$ defined in \eqref{eq:JacobianFormula} is Lipschitz continuous with $\overline{L}^{p, p^*}_{Dh^k}$ being an upper bound on the Lipschitz constant, where
        \begin{align*}
            \overline{L}^{p, p^*}_{Dh^k} = L_{\phi'} \|G^k\|_{p^*} \|W^k\|_{p^*} \|W^k\|_{p \to \infty},
        \end{align*}
        where $L_{\phi'} = \max\{|\alpha'|, |\beta'|\}$.
    \end{proposition}

        It is worth mentioning that the upper bound in \Cref{prop:naiveJacLip} is tractable and can be calculated efficiently.
        In particular, for $p = 2$, the matrix norms $\|G^k\|_{p^*}$ and $ \|W^k\|_{p^*}$ can be calculated via the power iteration for fully connected and convolutional layers.
        Furthermore, $\|W^k\|_{p \to \infty}$ is simply the maximum row $\ell_{p}$ norm, which is straightforward for fully connected layers and also convolutional layers with respect to the repetitive structure of their Toeplitz matrices. 
        See the proof in \Cref{par:proofJacLipLayer} for more details.
    
For the choice $p=p^*=2$, we propose an alternative approach to acquire better Lipschitz estimates for $Dh^k$. To this end, we propose to rewrite $Dh^k$ as a standard network block.
\begin{lemma}[Vectorized Jacobian]\label{prop:jacobianConversionToLinear}
    The Jacobian matrix in \eqref{eq:JacobianFormula} can be rewritten as a standard neural network layer
    \[
    dh^k(x^k) := \mathrm{vec}(Dh^k(x^k))=b^k + A^k\Phi'(W^k x^k),
    \]
    where $dh^k(x^k) \in \mathbb{R}^{\hat{n}_k}$ with $\hat{n}_k = n_{k + 1} \times n_{k}$. For all $i \in [n_{k + 1}]$, $ j \in [n_k]$, and $l \in [n^\prime_{k}]$  let $m = (j - 1)\times n_{k + 1} + i$. Then $A \in \mathbb{R}^{\hat{n}_k \times n^\prime_k}$ and $b^k \in \mathbb{R}^{\hat{n}_k}$ are given by  
    $
    b^k_m = H^k_{ij},\ A^k_{ml} = G^k_{il}W^k_{lj}.
    $
\end{lemma}
The following lemma establishes the relation between the Lipschitz constant of the Jacobian matrix ($L^p_{Dh^k}$) and its vectorized representation ($L^p_{dh^k}$) when $p=2$.

\begin{lemma}\label{prop:lipschitzOfLinearJacobian}
    Let $p=2$, and suppose $L^p_{dh^k}$ is the Lipschitz constant of the vectorized Jacobian function $x^k \mapsto dh^k(x^k)$ defined in \Cref{prop:jacobianConversionToLinear}. 
    Then $L^p_{dh^k}$ is a valid  Lipschitz constant for the Jacobian function $x^k \mapsto Dh^k(x^k)$.
\end{lemma}
We can improve the Lipschitz bound provided in \Cref{prop:naiveJacLip} using the previous lemmas. 

\begin{theorem}\label{prop:vectorizedLipschitzIsBetter} 
    For $p = 2$,  
     $\overline{L}^p_{dh^k} = L_{\phi'} \|A^k\|_2 \|W^k\|_2$ is a Lipcshitz constant for the Jaocbian matrix $Dh^k$. Furthermore, $\overline{L}^p_{dh^k} \leq \overline{L}^p_{Dh^k}$.
\end{theorem}
Leveraging the vectorized representation $dh^k$ of the Jacobian $Dh^k$, we can utilize more advanced techniques for Lipschitz estimation such as LipSDP to further reduce the conservatism. Specifically, for non-residual building blocks, i.e., when $H^k = 0$ and $G^k = I$, we can extract a feasible solution (optimal when $\alpha' = -\beta'$) to the LipSDP formulation for $dh^k(x)$. 
\begin{theorem}\label{prop:analyticalSdp}
    Let $h^k(x) = \Phi(W^kx)$. Define $\overline{L}^{2, \text{SDP}}_{dh^k}= L_{\phi'} \| T W^k \|_2$, where $T$ is a diagonal matrix with $T_{ii} = \| W^k_{i, :} \|_2$. Then $\overline{L}^{2, \text{SDP}}_{dh^k}$ is a valid Lipschitz constant for $dh^k$ in $\ell_2$ norm.
\end{theorem}






\subsubsection{Summary of Algorithm}
It now remains to combine all the components developed thus far to obtain upper bounds on the curvature of the whole network. This is summarized in \Cref{alg:iterativeCurvatureEstimation}.
First, we use LipLT to calculate the Lipschitz constants of the individual layers ($\overline{L}^{p^*}_{h^k}$) and the subnetworks ($\overline{L}^p_k$ and $\overline{L}^{p^*}_k$). Then, using \Cref{prop:naiveJacLip}, we provide an upper bound on $L^{p, p^*}_{Dh^k}$. Finally, we calculate $\overline{L}^{p, p^*}_{D^{k + 1}}$ using \eqref{eq: curvature bound for sequential networks}.

In the algorithm, we can easily swap the use of \Cref{prop:naiveJacLip} with any Lipschitz constant acquired based on the theoretical ground of \Cref{prop:lipschitzOfLinearJacobian}, such as that of \Cref{prop:vectorizedLipschitzIsBetter} or \Cref{prop:analyticalSdp} (if the network is non-residual).

\begin{algorithm}[tb]
   \caption{Compositional Curvature Estimation of Neural Networks}
   \label{alg:iterativeCurvatureEstimation}
\begin{algorithmic}
   \STATE {\bfseries Input:} $K$-layer neural network in the form of \eqref{eq:sequential model}.
   \STATE {\bfseries Initialize} $\overline{L}^p_0 = \overline{L}^{p^*}_0 = 1, \overline{L}^{p, p^*}_{D_0} = 0$.
   \FOR{$k=0$ {\bfseries to} $K - 1$}
   \STATE Calculate $\overline{L}^{p^*}_{h^k}$ using \eqref{eq:LipLT single layer}.
   \STATE Calculate a bound on $\overline{L}^{p, p^*}_{Dh^k}$ using \Cref{prop:naiveJacLip}.
   \STATE Update $\overline{L}^{p, p^*}_{D_{k+1}} = \overline{L}^{p, p^*}_{Dh^k} \overline{L}^p_{k} \overline{L}^{p^*}_{k} + \overline{L}^{p^*}_{h^k} \overline{L}^{p, p^*}_{D_k}$.
   \STATE Calculate $\overline{L}^p_{k + 1}$ and $\overline{L}^{p^*}_{k + 1}$ using \eqref{eq:lipLtRecursive}.
   \ENDFOR
   \STATE {\bfseries Return} $\overline{L}^{p, p^*}_{D_K}$, the Lipschitz constant of the Jacobian of $x^0 \mapsto x^K$. 
\end{algorithmic}
\end{algorithm}

\subsubsection{Comparison with existing approaches} 
Unlike the previous work by Singla et al. \cite{singla2020second}, which is limited to scalar-valued and non-residual architectures, our framework accommodates vector-valued general sequential models. Additionally, although Singla et al. \cite{singla2020second} could theoretically handle convolutional neural networks by expressing such layers as equivalent fully connected layers using their Toeplitz matrices \cite{CHEN2020}, this approach would be computationally prohibitive. In contrast, our method readily applies to convolutional layers.

Moreover, our method does \emph{not} require twice differentiability. This is particularly relevant for functions with Lipschitz continuous first derivatives but undefined second derivatives. For instance, consider the well-known Exponential Linear Unit (ELU): 
\[
f(z) = \begin{cases}
    z \qquad &z \geq 0 \\
    \alpha(e^z - 1) &z \leq 0
\end{cases}
\]
The ELU has a Lipschitz continuous first derivative, but its second derivative is not defined at \(z = 0\). Therefore, Hessian-based analysis would fail for this function, whereas our Jacobian Lipschitz analysis is applicable to networks using this activation function.

In the following section, we will utilize \Cref{alg:iterativeCurvatureEstimation} to bound the Lipschitz constant of the Jacobian and exploit it during the training phase to control the curvature of the neural network.

\subsection{Curvature-Controlled Networks}
As established in \Cref{sec:curvatureBasedAnalysis}, models with low curvature constants can elicit more robust behavior against norm-bounded perturbations. 
Driven by this observation, we can design a curvature-based regularizer that would promote robustness during training. One approach is to reward large certified radii in the objective similar to \cite{fazlyab2023certified,xu2023exploring}, giving rise to the training loss function
\begin{align*}
    \mathcal{L}(x, y; f) = \mathcal{L}_{\text{CE}}(f(x), u_y) +\lambda \mathrm{1}_{\{C(x)=y\}} g(\underline{\varepsilon}_1^*(x)),
\end{align*}
where $\mathcal{L}_\text{CE}$ is the cross-entropy loss, $u_y$ is the one-hot vector corresponding to $y$, $\lambda>0$ is the regularization constant, and $g \colon \mathbb{R}_{+} \to \mathbb{R}$ is a convex decreasing function (e.g., $g(z)=\exp(-z)$). The role of the indicator function $\mathrm{1}_{\{C(x)=y\}}$ is to restrict the regularizer to correctly classified points only. This regularizer is differentiable due to the closed-form expression for $\varepsilon_1^*(x)$ given in \eqref{eq:icmlLipschitzCertificate}. Nonetheless, computing it for each data point in the training dataset can be computationally costly for large-scale instances.
A more efficient approach is to regularize the bound on the global curvature of the network during training, 
\begin{align*}
    \mathcal{L}(x, y; f) = \mathcal{L}_{\text{CE}}(f(x), u_y) +\lambda \overline{L}^{p, p^*}_{Df},
\end{align*}
To further improve the efficiency, we propose to use 1-Lipschitz layers to build the architecture. Specifically, when $p=2$, if $h^k$ in \eqref{eq:sequential model} is modified to be a 1-Lipschitz function
($\overline{L}^p_{h^k}=1$), the concatenation of all layers will be 1-Lipschitz ($\overline{L}^p_{k}=1$), and thus \eqref{eq: curvature bound for sequential networks} yields the curvature bound $\overline{L}^{p}_{Df} = \sum_{k=0}^{L-1} \overline{L}^{p}_{Dh^k}$, which can be readily computed.

\section{Experiments}

\begin{figure}[t!]
            \centering
            \includegraphics[width=0.8\columnwidth]{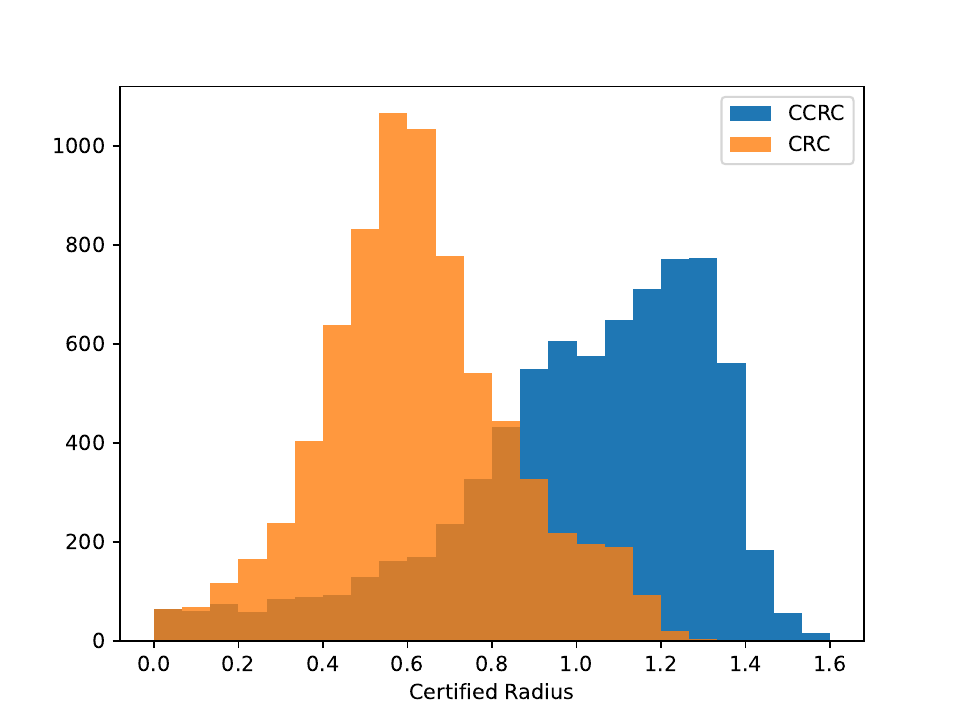}
            \caption{Certified radius comparison on a 6-layer neural network.}
            \label{fig:certified_radius}
\end{figure}

    
    
In this section, we evaluate the performance of our proposed methods via a series of experiments. We contrast our methods against the state-of-the-art Lipschitz constant estimation, curvature estimation, and neural network robustness certification algorithms. In our experiments we set $p = 2$ for all norms and Lipschitz calculations.
We defer the discussion of hyperparameters to \Cref{appendix:hyperparameters}.
Our code is available at \href{https://github.com/o4lc/Compositional-Curvature-Bounds-for-DNNs}{https://github.com/o4lc/Compositional-Curvature-Bounds-for-DNNs}.

We first showcase the application and superiority of our Jacobian Lipschitz constant estimation on MNIST \cite{lecun1998gradient}. 
Next, to further motivate the use of anchored Lipschitz constant estimation, we train several networks with different depths to portray its effectiveness on both Lipschitz constant estimation and Jacobian Lipschitz constant estimation. 
Finally, we compare our robustness certification method with the state-of-the-art classification on the CIFAR \cite{krizhevsky2009learning} dataset and provide attack certificates on the same networks.
    \paragraph{Comparison with other Curvature-based Methods}\label{para:curvatureComparison}
    In this experiment, we compare our proposed compositional curvature calculation method with the previous works.
    Consequently, we train a 6-layer fully connected network on MNIST with curvature regularization and compute the certified radii of the test data points for this network using two curvature calculation algorithms. 
    We denote \cite{singla2020second} as \textit{Curvature-based Robustness Certificate (CRC)}, and our method as \textit{Compositional Curvature-based Robustness Certificate (CCRC)}. 
    Next, to focus the experiment on comparing the curvature bounds, we use the method of \cite{singla2020second} to obtain the certified radius for each point.
    
    \Cref{fig:certified_radius} compares the certified radii of these methods and confirms the superior performance of the compositional curvature calculation algorithm.

\begin{table}[t!]
  \centering
  \caption{Comparison of certified accuracies obtained from state-of-the-art methods SLL \cite{araujo2023unified} and CRM \cite{fazlyab2023certified} on CIFAR-10.}
  \resizebox{0.49\textwidth}{!}
    {\footnotesize
    \begin{tabular}{lcccccc}
        \toprule
        \multicolumn{1}{c}{\multirow{2}[2]{*}{\textbf{Model}}} & 
        \multicolumn{1}{c}{\multirow{2}[2]{*}{\textbf{Methods}}} & 
        \multicolumn{1}{c}{\multirow{2}[2]{*}{\textbf{Accuracy}}}& 
        \multicolumn{3}{c}{\textbf{Certified Accuracy ($\varepsilon$)}} & 
        \multicolumn{1}{c}{\multirow{2}[2]{*}{\textbf {Parameters}}}\\
       \cmidrule{4-6}
          & & & $\frac{36}{255}$ & $\frac{72}{255}$ & $\frac{108}{255}$  & \\
        \midrule
        \multicolumn{1}{c}{\multirow{3}[2]{*}{\textbf{6C2F}}}
          & Standard & \textbf{79.95} & 0 & 0 & 0 &  0.7M\\
          &CRM & 58.57 &	36.25 & 18.36 & 7.37 &  0.7M\\
          &CCRC (Ours) & 61.15 & \textbf{49.53} & \textbf{33.36} & \textbf{16.95} & 0.7M\\
        \midrule
        \multicolumn{1}{c}{\multirow{2}[2]{*}{\textbf{Lip-3C1F}}}
          &SLL & \textbf{57.2} & 45.0 & 35.0 & 26.5 & 1M \\
          &SLL + CCRC (Ours) & 53.2 & \textbf{46.6} & \textbf{39.3} & \textbf{31.6} & 1M \\
        \midrule
        \multicolumn{1}{c}{\multirow{3}[2]{*}{\textbf{6F}}}
          &Standard & 61.89 & 0 & 0 & 0 &  4M \\
          &CRM & 60.63 &	42.73 & 24.75 & 12.6 & 4M\\
          &CCRC (Ours) & \textbf{62.1} & \textbf{52.09} & \textbf{40.8} &\textbf{ 29.17} & 4M \\ 
        \bottomrule
    \end{tabular}%
    }
  \label{table:main_results}%
\end{table}%

\paragraph{Anchored Lipschitz/Curvature Estimation}
    Next, we study the impact of localizing the computations via the concept of anchored Lipschitz constant introduced in this paper. 
To achieve this, we train fully connected neural networks of varying depths and calculate upper bounds on the Lipschitz and curvature constants of the network, both globally and in an anchored manner. \Cref{fig:AnchoredExperiment} illustrates the results on the MNIST dataset. For the anchored bounds, we average the values over the test dataset. The results demonstrate that using the anchored counterparts significantly improves the bounds.
    

\begin{figure}[t]
    \centering
     \includegraphics[width=1\columnwidth]{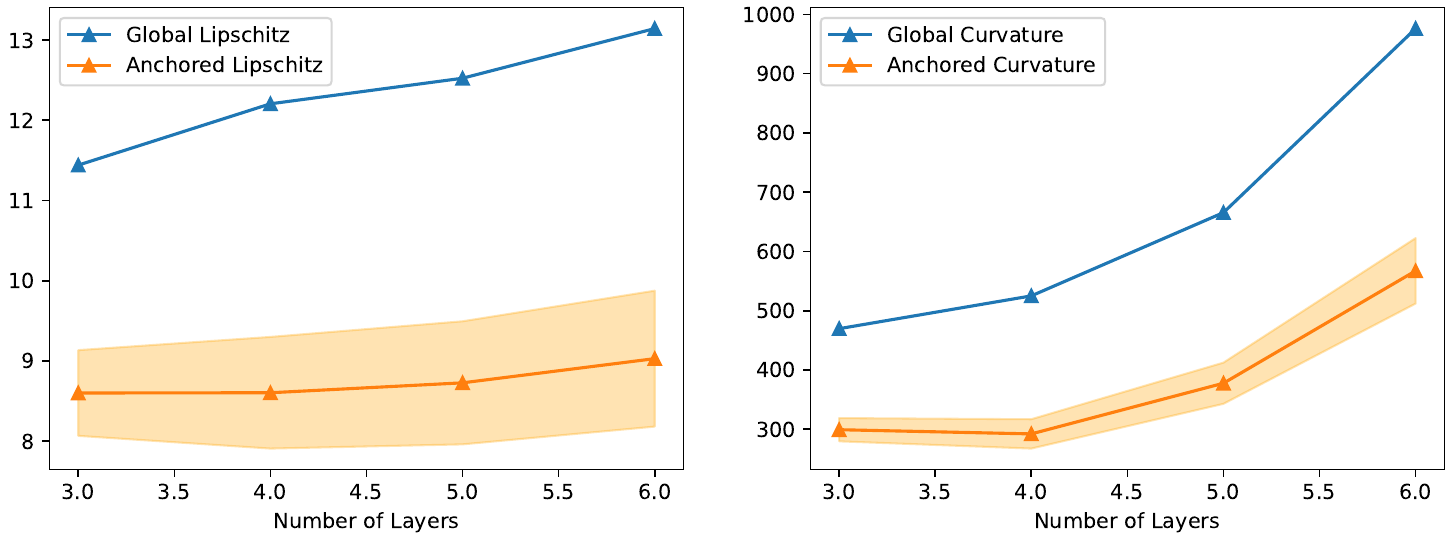}
    \caption{Comparison of global Lipschitz and Curvature estimation against their anchored counterparts. The shaded areas denote the standard deviation over the whole dataset.}
    \label{fig:AnchoredExperiment}
\end{figure}

\paragraph{Attack Certification on CIFAR-10}
\begin{figure}[t!]
    \centering
    \includegraphics[width=.8\columnwidth]{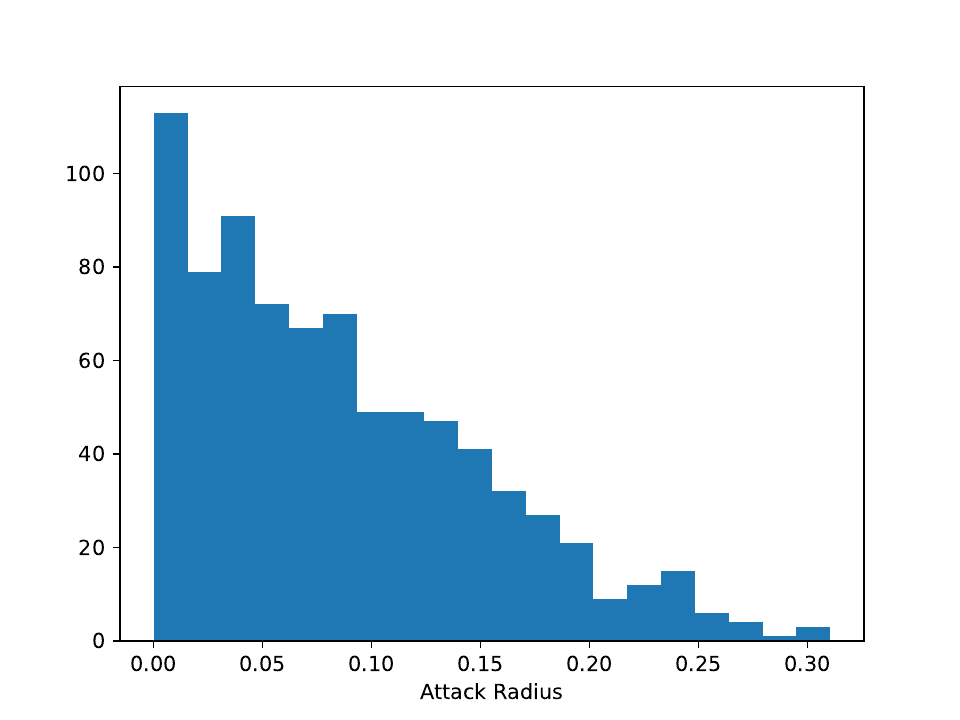}
    \caption{
    Histogram of the certified attack radii for a 6-layer neural network trained via curvature regularization on CIFAR-10.
    }
    \label{fig:newCertifiedAttackRadius}
\end{figure}

In this experiment, we provide radii for provable attacks on a 6F model. \Cref{fig:newCertifiedAttackRadius} shows the results. This model has an accuracy of $56.19\%$. Furthermore, with the perturbation budget  $\frac{36}{255}$ the model has certified and PGD accuracy of $47.16\% $ and $48.46\%$, respectively.
By analyzing the attack certificates we find that our method is able to provide an attack certificate for a total of $808$ samples, of which $645$ require a perturbation budget of at most $\frac{36}{255}$. 
Using this information, the robust accuracy of the model with this perturbation budget is at most $56.19 - \frac{645}{10000}\times 100 = 49.74\%$. 
This is illustrated in \Cref{fig:attack_example}, where $A_c$ is the clean accuracy,  $A^*_v $ is the verified accuracy, and $ \underline{A}_v$ and $\overline{A}_v$, are  lower and upper bounds on the verified accuracy, respectively.

This has two main implications.
First, having attack certificates for any data eliminates the need to perform an attack on that data as the existence of an attack was verified by our proposition. 
Second, the attack certificates further narrow down the uncertainty of the model accuracy. 
As the certified accuracy is a lower bound on the actual certified robustness of the model, we conclude that the actual certified accuracy of this model is in the range $[47.16, 49.74]$, \textit{regardless} of the certification method. 

\begin{figure}[t]
    \centering
    \includegraphics[width=.85\columnwidth]{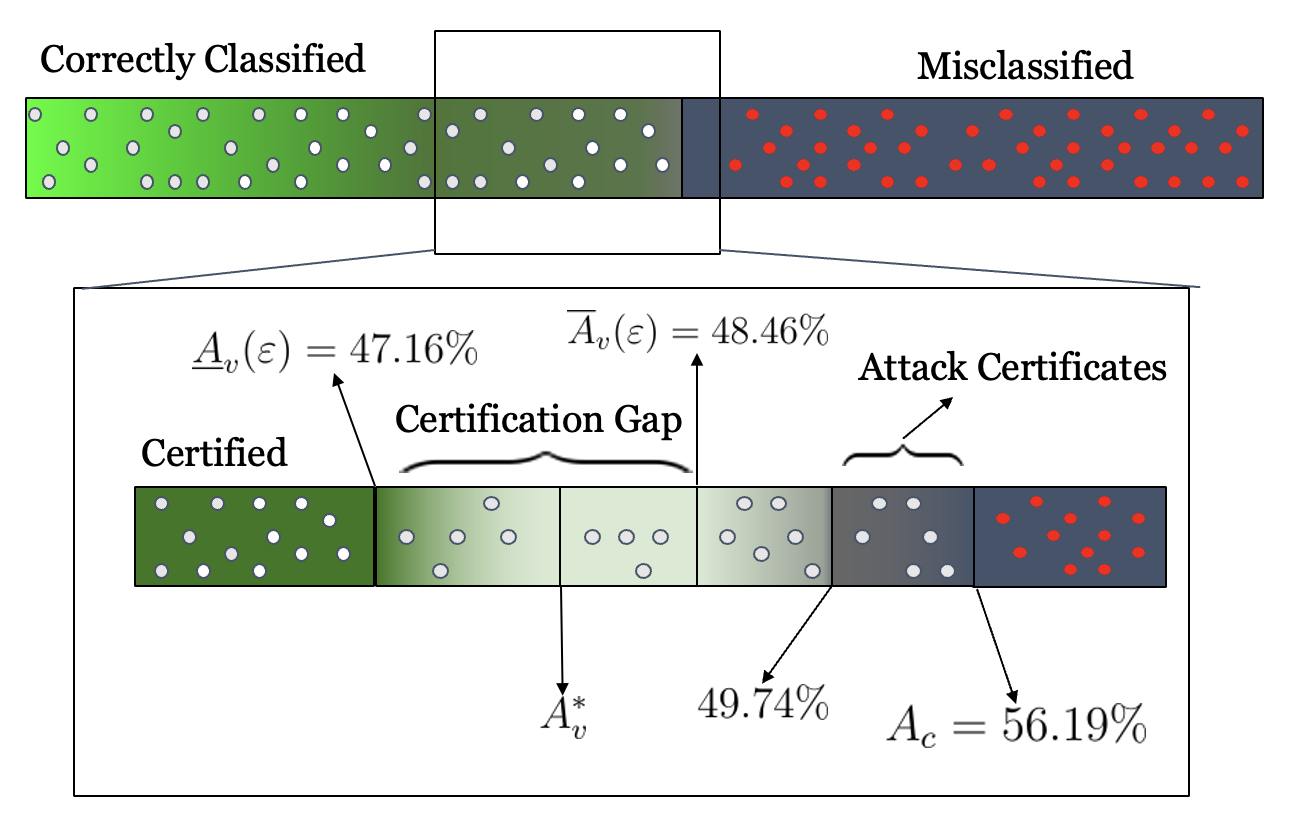}
    \caption{Illustration of the certificates provided by our method. $A_c$ and $A_v$ are the clean and verified accuracies, respectively.
    The certification radius is $\varepsilon = \frac{36}{255}$.}
    \label{fig:attack_example}
\end{figure}

\paragraph{Robustness Certification on CIFAR-10}
The final experiment aims to compare the certified accuracy of models with different architectures on CIFAR-10 with the state-of-the-art.
We train two 6C2F and 6F non-residual and a 1-Lipschitz neural network with \textit{Lip-3C1F} architecture on the CIFAR-10 dataset with the following loss function
\begin{equation}\label{eqn:regularizationOfLipschitz}
    \mathcal{L}(x, y; f) = \mathcal{L}^{\tau, \nu}_{\text{CE}}(f(x), y) + \lambda \overline{L}^p_{Df},
\end{equation}
where $\mathcal{L}^{\tau, \nu}_{\text{CE}}$ is a modified variant of the  cross entropy loss function \cite{prach2022almost} and $\overline{L}^p_{Df}$ is the curvature bound acquired through \Cref{alg:iterativeCurvatureEstimation}. 
Refer to \Cref{appendix:hyperparameters} for more on the loss function details.
\Cref{table:main_results} shows the comparison between these models.
We find that incorporating the additional regularization term leads to higher certified accuracies, smaller certification gaps, and often, higher clean accuracies.

\section{Conclusion}
In this work, we proposed a novel method to calculate provable upper bounds on the curvature constant of smooth deep neural networks, i.e., the Lipschitz constant of their first derivative.  Our method leverages the compositional structure of the model to compute the curvature bound in a scalable and modular fashion. The generality of our curvature estimation algorithm can enable its use for compositional functions beyond neural networks. Furthermore, we provided analytical robustness certificates for deep classifiers based on the curvature of the model. In the future, we aim to further tighten the estimated gap of the curvature bound, enabling the algorithm to produce tighter bounds on the curvature of even deeper neural networks.

\section*{Impact Statement}
This paper presents work whose goal is to advance the field of Machine Learning. 
We do not foresee any societal implications arising solely from our work.



\bibliography{bib}
\bibliographystyle{icml2024}

\newpage
\appendix
\onecolumn
\section{Theorems and Proofs}
\paragraph{Proof of \Cref{lemma:lipschitzUpperBoundOnNormJacobian}}
\begin{proof}
    For any direction $d \in \mathbb{R}^n$, let $g_d(t) = f(x + td)$. Evidently, we have $g_d'(t) = \frac{d}{dt}g_d(t) = Df(x+td)d$, where $Df(x) \in \mathbb{R}^{m \times n}$. Moreover, we have
    \begin{align*}
        \|Df(x)d\|_p=\| g_d'(0) \|_p &= \| \lim_{t \to 0} \dfrac{g_d(t) - g_d(0)}{t - 0} \|_p
        = \lim_{t \to 0} \dfrac{\|f(x + td) - f(x)\|_p}{t}
        \leq L_f^{p}(x) \|d\|_p,
    \end{align*}
    where the third equality follows from continuity of $\|\cdot\|_p$. Consequently, 
    \[
    \|Df(x)\|_p = \sup_{\|d\|_p \leq 1} \|Df(x)d\|_p = \sup_{\|d\|_p \leq 1} \|g_d'(0)\|_p \leq L_f^p(x).
    \]
\end{proof}
\begin{corollary}\label{lemma:lipschitzAndGradientRelation}
    Consider a differentiable function $f: \mathbb{R}^n \to \mathbb{R}$ and let $L_f^p(x)$ be a corresponding anchored Lipschitz constant in some $p$-norm. We have 
    \[
    \| \nabla f(x)\|_{p^*} \leq L_f^p(x).
    \]
\end{corollary}
\begin{proof}
    The proof follows from the fact that $Df(x) = \nabla f(x)^\top$ and that for matrix $A$ and norm $\|\cdot\|_p$, we have $\|A^\top\|_p = \|A\|_{p^*}$.
\end{proof}

\paragraph{Derivation of equation \eqref{eq:descentLemma}}
\phantomsection\label{par:proopOfSecondOrderRadius}
\begin{proof}
    To prove the quadratic upper bound on the logit gap with the anchored Lipschitz constant we utilize the mean value theorem.
   For any $x$ and $\delta$, we can write
    \begin{align*}
        f_{iy}(x+\delta)
        =  f_{iy}(x) + \nabla f_{iy}(x)^\top \delta +  \int_{0}^{1} (\nabla f_{iy}(x+t\delta)-\nabla f_{iy}(x))^\top \delta \: dt.
    \end{align*}
    We can then write
    \begin{align*}
            \begin{split}
                |f_{iy}(x+\delta)-f_{iy}(x)-\nabla f_{iy}(x)^\top \delta| &\leq  \int_{0}^{1} |(\nabla f_{iy}(x+t\delta)-\nabla f_{iy}(x))^\top \delta| dt \\
            \xrightarrow[\frac{1}{p^*} + \frac{1}{p} = 1]{\text{H$\mathrm{\ddot{o}}$lder's Inequality}} & \leq \int_{0}^{1} \|(\nabla f_{iy}(x+t\delta)-\nabla f_{iy}(x)\|_{p^*} \|\delta\|_p dt \\
            &\leq \int_{0}^{1} L_{\nabla f_{iy}}^{p, {p^*}}(x) \|\delta\|_p^2 \: t \: dt \label{eq:descentLemmaLipschitzInequality} \\
            &= \frac{L_{\nabla f_{iy}}^{p, {p^*}}(x)}{2} \|\delta\|_p^2.
            \end{split}
    \end{align*}
    
    Thus, we have 
    \begin{equation*}\label{eq:quadraticUpperBound}
        f_{iy}(x+\delta) \leq  f_{iy}(x) + \nabla f_{iy}(x)^\top \delta + \frac{L_{\nabla f_{iy}}^{p, {p^*}}(x)}{2} \|\delta\|_p^2.
    \end{equation*}
\end{proof}

\paragraph{Proof of \Cref{prop:secondOrderCertifiedRadius}}
\phantomsection\label{par:proofOfCertRad}
\begin{proof}
    Define 
    $\Delta = \{ \delta \: | \: f_{iy}(x + \delta) \leq 0 \}$ 
    and 
    $\overline{\Delta} = \{ \delta \: | \: \overline{f_{iy}}(x, \delta; L_{\nabla f_{iy}}^{p, {p^*}}(x)) \leq 0\}$.
    It is clear that $\overline{\Delta} \subseteq \Delta$.
    Consequently, $\{ \varepsilon \mid  \sup_{\| \delta \|_p \leq \varepsilon} \overline{f_{iy}}(x, \delta; L_{\nabla f_{iy}}^{p, {p^*}}(x)) \leq 0\}  \subseteq \{ \varepsilon \mid  \sup_{\| \delta \|_p \leq \varepsilon} f_{iy}(x + \delta) \leq 0\}$, and thus, \eqref{eq:defenseRadiusRelaxed1}  yields a valid lower bound on the original radius \eqref{eq:defenseRadius}.
    
    Now, to acquire the analytical result we have
    \begin{align}\label{eq:upperBoundSumSup}
        \sup_{\|\delta\|_p \leq \varepsilon} f_{iy}(x) + \nabla f_{iy}(x)^\top \delta + \frac{L_{\nabla f_{iy}}^{p, {p^*}}(x)}{2} \|\delta\|_p^2
        &= f_{iy}(x) + \varepsilon \|\nabla f_{iy}(x)\|_{p^*} + \frac{L_{\nabla f_{iy}}^{p, {p^*}}(x)}{2} \varepsilon^2. \notag
    \end{align}
    This equality holds due to the fact that for the $\delta^*$ achieving $\sup_{\|\delta\|_p \leq \varepsilon} \nabla f_{iy}(x)^\top \delta $ we have $\|\delta^*\|_p = \varepsilon$, implying that $\delta^*$ is also a maximizer of the other term $\frac{L_{\nabla f_{iy}}^{p, {p^*}}(x)}{2} \|\delta\|_p^2$.

    As a result, we obtain the following optimization problem that is equivalent to \eqref{eq:defenseRadiusRelaxed1}
    \begin{align*}
        \max &\quad \varepsilon \\
        \text{s.t.} &\quad \frac{L_{\nabla f_{iy}}^{p, {p^*}}(x)}{2} \varepsilon^2 + \| \nabla f_{iy}(x) \|_{p^*} \varepsilon + f_{iy}(x) \leq 0, \forall i \neq y.
    \end{align*}
    Using elementary calculations, we obtain the optimal solution
    \[
    \min_{i\neq y} \frac{-\| \nabla f_{iy}(x) \|_{p^*} \! + \! (\| \nabla f_{iy}(x) \|_{p^*}^2 \! - \! 2 L_{\nabla f_{iy}}^{p, {p^*}}(x) f_{iy}(x))^{\frac{1}{2}}}{L_{\nabla f_{iy}}^{p,{p^*}}(x)}.
    \]
\end{proof}
\paragraph{Proof of \Cref{prop:compare}}
\phantomsection\label{par:radiiComp}
\begin{proof}
    Suppose $x$ is correctly classified, i.e., $f_{iy}(x) <0$ for $i \neq y$. To enforce the condition $\underline{\varepsilon}_0^*(x) \leq \underline{\varepsilon}_1^*(x)$ we must have
    \begin{align}
        \underline{\varepsilon}_0^*(x) \leq
        \frac{-\| \nabla f_{iy}(x) \|_{p^*} \! + \! (\| \nabla f_{iy}(x) \|_{p^*}^2 \! - \! 2 L_{\nabla f_{iy}}^{p, {p^*}}(x) f_{iy}(x))^{\frac{1}{2}}}{L_{\nabla f_{iy}}^{p, {p^*}}(x)}, \qquad \forall i \neq y. \notag
    \end{align}
    Or equivalently, 
    \begin{align}
        L_{\nabla f_{iy}}^{p, {p^*}}(x) \underline{\varepsilon}_0^*(x)^2 + 2\| \nabla f_{iy}(x) \|_{p^*} \underline{\varepsilon}_0^*(x) +2 f_{iy}(x)\leq 0. \notag
    \end{align}
    The above inequality holds if and only if
    \begin{align}\label{eq:curvatureBetterThanLipschitzCondition}
        L_{\nabla f_{iy}}^{p, {p^*}}(x) \leq \frac{-2(\| \nabla f_{iy}(x) \|_{p^*}\underline{\varepsilon}_0^*(x) + f_{iy}(x))}{\underline{\varepsilon}_0^*(x)^2}.
    \end{align}
    Utilizing \Cref{lemma:lipschitzAndGradientRelation}, we note that  $ \underline{\varepsilon}_0^*(x) = \min_{i \neq y} \frac{-f_{iy}(x)}{L_{f_{iy}}^p(x)} \leq \frac{-f_{iy}(x)}{L_{f_{iy}}^p(x)} \leq \frac{-f_{iy}(x)}{\| \nabla f_{iy}(x) \|_{p^*}}$.
    This ensures that the R.H.S. of \eqref{eq:curvatureBetterThanLipschitzCondition} is positive.
    
\end{proof}

\paragraph{Proof of \Cref{prop:equivalenceOfAttackAndDefence}}
\begin{proof}
    We prove this in two steps:
    \begin{enumerate}
        \item ${\varepsilon}^*(x) \leq {\varepsilon^{'}}^*(x)$: Suppose this does not hold. Then $\varepsilon^*(x) > {\varepsilon^{'}}^*(x)$. However, having a solution for \eqref{eq:attackRadius} implies that there exists a pair $\delta$ and $j$ with $\| \delta \|_p \leq {\varepsilon^{'}}^*(x) < \varepsilon^*(x)$ such that $f_{yj}(x + \delta) < 0$. This perturbation-index pair is then a violation for the constraint of problem \eqref{eq:defenseRadius}. Thus, we must have ${\varepsilon}^*(x) \leq {\varepsilon^{'}}^*(x)$.
        \item ${\varepsilon}^*(x) \geq {\varepsilon^{'}}^*(x)$: Similarly, if this does not hold, then $\forall i\neq y, \delta, \| \delta \|_p \leq {\varepsilon}^*(x) < {\varepsilon^{'}}^*(x)$ we have 
        $\sup_{\| \delta \|_p \leq \varepsilon^*(x)} f_{iy}(x + \delta) \leq 0$. But having a solution for \eqref{eq:attackRadius} asserts that there exists one such index that $f_{yi}(x + \delta) < 0$. Thus, ${\varepsilon}^*(x) \geq {\varepsilon^{'}}^*(x)$ must hold.
    \end{enumerate}
    As a result, we must have ${\varepsilon}^*(x) = {\varepsilon^{'}}^*(x)$.
\end{proof}

\paragraph{Proof of \Cref{prop:AttackCert}}
\begin{proof}
    To prove this proposition, we first find the analytical solution to the inner optimization problems, namely
    \begin{equation}\label{eq:infimumAttackEquation}
        \inf_{\| \delta \|_p \leq \varepsilon} \overline{f_{yi}}(x, \delta; L_{\nabla f_{iy}}^{p, {p^*}}(x)) = \inf_{0 \leq \lambda \leq 1} \inf_{\| \delta \|_p = \lambda \varepsilon} f_{yi}(x) + \nabla f_{yi}(x)^\top \delta + \frac{L_{\nabla f_{yi}}^{p, {p^*}}(x)}{2} \|\delta\|_p^2.
    \end{equation}
    Let $\delta^* = \arg\min_{\|\delta\|_p\leq 1} \nabla f_{yi}(x)^\top \delta$, which yields $\nabla f_{yi}(x)^\top \delta^* = -\|\nabla f_{yi}(x)\|_{p^*}$.
    The inner optimization problem is minimized at $\delta = \lambda \varepsilon \delta^*$.
    Consequently, we can frame problem \eqref{eq:infimumAttackEquation} equivalently as 
    \begin{equation}
        g_{i}^*(\varepsilon) = g_{i}(\varepsilon, \lambda^*) = \min_{0 \leq \lambda \leq 1} \underbrace{f_{yi}(x) -  \varepsilon \|\nabla f_{yi}(x)\|_{p^*} \lambda + \frac{L_{\nabla f_{yi}}^{p, {p^*}}(x)}{2} \varepsilon^2 \lambda^2}_{g_{i}(\varepsilon, \lambda)}.
    \end{equation}
    There are three possible solutions 
    \begin{enumerate}
        \item $\hat{\lambda} = 0$. In this scenario
        \[
        g_{i}(\varepsilon, \hat{\lambda}) = f_{yi}(x).
        \]
        \item $\hat{\lambda} = \dfrac{\|\nabla f_{yi}(x)\|_{p^*}}{L_{\nabla f_{yi}}^{p, {p^*}}(x) \varepsilon}$. 
        In this scenario
        \[
        g_{i}(\varepsilon, \hat{\lambda}) = f_{yi}(x) - \dfrac{\|\nabla f_{yi}(x)\|_{p^*}^2}{2L_{\nabla f_{yi}}^{p, {p^*}}(x)}.
        \]
        For this solution we must have $\hat{\lambda} \leq 1$. This imposes the condition $\dfrac{\|\nabla f_{yi}(x)\|_{p^*}}{L_{\nabla f_{yi}}^{p, {p^*}}(x)} \leq \varepsilon$.
        \item $\hat{\lambda} = 1$. In this scenario we have
        \[
        g_{i}(\varepsilon, \hat{\lambda}) = f_{yi}(x) - \varepsilon \|\nabla f_{yi}(x)\|_{p^*} + \frac{L_{\nabla f_{yi}}^{p, {p^*}}(x)}{2} \varepsilon^2.
        \]
    \end{enumerate}
    Putting together the different conditions, we find that
    \[
    g_{i}^*(\varepsilon) = 
    \begin{cases}
        f_{yi}(x) - \dfrac{\|\nabla f_{yi}(x)\|_{p^*}^2}{2L_{\nabla f_{yi}}^{p, {p^*}}(x)} 
        &,  \dfrac{\|\nabla f_{yi}(x)\|_{p^*}}{L_{\nabla f_{yi}}^{p, {p^*}}(x)} \leq \varepsilon\\
        f_{yi}(x) - \varepsilon \|\nabla f_{yi}(x)\|_{p^*} + \dfrac{L_{\nabla f_{yi}}^{p, {p^*}}(x)}{2} \varepsilon^2 
        &, \dfrac{\|\nabla f_{yi}(x)\|_{p^*}}{L_{\nabla f_{yi}}^{p, {p^*}}(x)} > \varepsilon
    \end{cases}
    \]
    The next step is solving 
    \begin{align*}
        \min &\quad \varepsilon\\
        \text{s.t.} &\quad \min_{i\neq y} g_i^*(\varepsilon) < 0.
    \end{align*}
    For each $i \neq y$, if $\dfrac{\|\nabla f_{yi}(x)\|_{p^*}}{L_{\nabla f_{yi}}^{p, {p^*}}(x)} \leq \varepsilon$ then $g_i^*(\varepsilon) \leq 0$
    requires 
    $2 L_{\nabla f_{yi}}^{p, {p^*}}(x) f_{yi}(x) \leq \| \nabla f_{yi}(x) \|_{p^*}^2$,
    and if $\dfrac{\|\nabla f_{yi}(x)\|_{p^*}}{L_{\nabla f_{yi}}^{p, {p^*}}(x)} > \varepsilon$
    the smallest value of $\varepsilon$ yielding $g_i^*(\varepsilon) \leq 0$ is 
    \begin{equation}\label{eq:singleAttackRadius}
        \frac{\|\nabla f_{yi}(x)\|_{p^*} - \sqrt{\|\nabla f_{yi}(x)\|_{p^*}^2 - 2L_{\nabla f_{yi}}^{p, {p^*}}(x) f_{yi}(x)}}{L_{\nabla f_{yi}}^{p, {p^*}}(x)},
    \end{equation}
    which similarly requires the condition $2 L_{\nabla f_{yi}}^{p, {p^*}}(x) f_{yi}(x) \leq \| \nabla f_{yi}(x) \|_{p^*}^2$ for realizability.
    Consequently, if $2 L_{\nabla f_{yi}}^{p, {p^*}}(x) f_{yi}(x) \leq \| \nabla f_{yi}(x) \|_{p^*}^2$ holds for some $i\neq y$, the smallest valid $\varepsilon$ is given as in \eqref{eq:singleAttackRadius}. 
    Thus, we conclude that 
    \begin{equation}
        \overline{\varepsilon}_1^*(x) = \min_{i \in \mathcal{I}}  \frac{\|\nabla f_{yi}(x)\|_{p^*} - \sqrt{\|\nabla f_{yi}(x)\|_{p^*}^2 - 2L_{\nabla f_{yi}}^{p, {p^*}}(x) f_{yi}(x)}}{L_{\nabla f_{yi}}^{p, {p^*}}(x)},
    \end{equation}
    where $\mathcal{I} = \{ i | i \neq y, 2 L_{\nabla f_{yi}}^{p, {p^*}}(x) f_{yi}(x) \leq \| \nabla f_{yi}(x) \|_{p^*}^2   \}$. If $\mathcal{I} = \varnothing$, the problem is infeasible.

\end{proof}

\paragraph{Proof of \Cref{prop:recursiveJacobianLipschitz}}
\phantomsection\label{par:proofCompCurv}
\begin{proof}
    Writing the definition of Lipschitz continuity, we have
    \begin{align*}
        \begin{split}
            \| Dh(x) - Dh(y) \|_q &= \| Df(g(x))Dg(x) - Df(g(y))Dg(y) \|_q\\
            &= \| Df(g(x))Dg(x) - Df(g(x))Dg(y)  + Df(g(x))Dg(y) - Df(g(y))Dg(y) \|_q\\
             &\leq \|Df(g(x))Dg(x) - Df(g(x))Dg(y)\|_q  + \| Df(g(x))Dg(y) - Df(g(y))Dg(y) \|_q\\
             &\leq \|Df(g(x)) \|_q \|Dg(x) - Dg(y)\|_q  + \| Df(g(x)) - Df(g(y))\|_q \|Dg(y) \|_q\\
             &\leq L_{Dg}^{p, q} \|Df(g(x)) \|_q \| x - y \|_p + L_{Df}^{q', q} \|g(x) - g(y)\|_{q'}\| Dg(y) \|_q\\
             &\leq L_{Dg}^{p, q} \|Df(g(x)) \|_q \| x - y \|_p + L_{Df}^{q', q} L_g^{p, q'} \| Dg(y) \|_q  \| x - y\|_p.
        \end{split}
    \end{align*}
    Based on \Cref{lemma:lipschitzUpperBoundOnNormJacobian}, we note that $L^q_f$ is an upper bound on $\| Df(\cdot) \|_q$ and that $L^q_{g}$ is an upper bound on $\| Dg(\cdot) \|_q$. Thus, we arrive at
    \[
    \| Dh(x) - Dh(y) \|_q \leq \big( L^{p,q}_{Dg} L^q_f  + L^{q', q}_{Df} L_g^{p,q'} L_g^{q} \big) \|x - y\|_p.
    \]
    Finally, by setting $q = p^*$ and $q' = p$, we obtain
    $
    L_{Dh}^{p, p^*} \leq L^{p,p^*}_{Dg} L^{p^*}_f  + L^{p, p^*}_{Df} L_g^{p} L_g^{p^*}$. 
\end{proof}

\paragraph{Proof of \Cref{prop:anchoredIterativeJacobian}}

\begin{proof}
    Taking the same steps as the proof of \Cref{prop:recursiveJacobianLipschitz}, we have
    \begin{align*}
        \begin{split}
            \| Dh(x + \delta) - Dh(x) \|_q 
             &\leq  \|Df(g(x + \delta)) \|_q \|Dg(x + \delta) - Dg(x)\|_q + \| Df(g(x + \delta)) - Df(g(x))\|_q \|Dg(x) \|_q\\
             &\leq L^{p,q}_{Dg}(x) \|Df(g(x + \delta)) \|_q \| \delta \|_p + L^{q', q}_{Df}(g(x)) L^{p, q'}_g(x) \| Dg(x) \|_q  \| \delta \|_p\\
             &\leq \big(L^q_{f} L^{p,q}_{Dg}(x) + \|Dg(x)\|_q L^{q', q}_{Df}(g(x)) L^{p, q'}_g(x) \big) \|\delta \|_p.
        \end{split}
    \end{align*}
    Setting $q = p^*$ and $q' = p$ yields the result.
\end{proof}

\paragraph{Proof of \Cref{prop:naiveJacLip}}
\phantomsection\label{par:proofJacLipLayer}
\begin{proof}
        We drop the superscript $k$ for simplicity here.
        Writing out the definition of Lipschitz continuity we have
        \begin{align}\label{eq:proofNaiveJacLip}
            \begin{split}
                \| Dh(x) - Dh(y) \|_q 
                &= \| G \: \mathrm{diag}(\Phi'(Wx)) W - G \: \mathrm{diag}(\Phi'(Wy)) W \|_q\\
                &\leq \| G \|_q \|W\|_q \| \mathrm{diag}(\Phi'(Wx)) - \mathrm{diag}(\Phi'(Wy)) \|_q\\
                &= \| G \|_q \|W\|_q \max_i | \Phi'(Wx)_i - \Phi'(Wy)_i|\\
                &\leq L_{\phi'}  \| G \|_q \|W\|_q \max_i | W_{i, :}(x - y)|,
            \end{split}
        \end{align}
        where $L_{\phi'} = \max\{|\alpha'|, |\beta'|\}$ is the Lipschitz constant of $\phi'$.
        Next, we can write
            \begin{align*}
                \max_i | W_{i, :}(x - y)| &= \|W(x - y)\|_\infty \leq \| W \|_{p \to \infty} \|x - y\|_p,
            \end{align*}
            Setting $q = p^*$ yields the desired result.

    \end{proof}
\paragraph{Proof of \Cref{prop:jacobianConversionToLinear}}
\begin{proof}
    To perform the conversion to a standard layer, we consider individual entries of the output:
    \[
    Dh^k(x)_{ij} = H^k_{ij} + \sum_{l=1}^{n'_k} G^k_{il}W^k_{lj}\Phi'(W^kx)_l.
    \]
    Thus, by flattening the matrix $Dh^k(x)$ into a vector $dh^k(x)$, where the $ij$-th element of $Dh^k(x)$ is mapped to the $\big((j - 1) \times n_{k + 1} + i\big)$-th element of $dh^k(x)$, $\forall i \in [n_{k+1}],  j \in [n_k]$, we obtain the desired result. The vector $b^k$ and matrix $A^k$ defined in the lemma yield the correct map.
    Importantly, we have the identity $A^k = \hat{G}^{k} \tilde{W}^k$, where
    \[
    \hat{G}^k = \begin{bmatrix}
        G^k & 0 & \cdots & 0\\
        0 & G^k & \cdots & 0\\
        \vdots & \vdots & \ddots & \vdots\\
        0 & 0 & \cdots & G^k
    \end{bmatrix},
    \quad 
    \tilde{W}^k = \begin{bmatrix}
        \mathrm{diag}(W^k_{:, 1})\\
        \mathrm{diag}(W^k_{:, 2})\\
        \vdots\\
        \mathrm{diag}(W^k_{:, n_{k}})\\
    \end{bmatrix}.
    \]
    Evidently, we have $A^k_{ml} = \big( G^k \mathrm{diag}(W^k_{:, j}) \big)_{il} = G^k_{il} W^k_{lj}$.
\end{proof}

\paragraph{Proof of \Cref{prop:lipschitzOfLinearJacobian}}
\begin{proof}
    By vectorization, we have $\|dh^k(x) - dh^k(y)\|_2=\| Dh^k(x) - Dh^k(y)\|_F$, where $\|\cdot\|_F$ is the Frobenius norm. We can write
    \begin{align*}
        \begin{split}
            \|Dh^k(x) - Dh^k(y)\|_2 &\leq \|Dh^k(x) - Dh^k(y)\|_F = \|dh^k(x) - dh^k(y)\|_2 \leq L^p_{dh^k} \|x - y\|_2,
        \end{split}
    \end{align*}
    where we have used the fact that for a given matrix $A$, $\|A\|_2 = \sigma_{\text{max}}(A) \leq \sqrt{\sum_i \sigma_i^2(A)} = \|A\|_F$.
\end{proof}


\paragraph{Proof of \Cref{prop:vectorizedLipschitzIsBetter}}
\begin{proof}
    Using identity $A^k = \hat{G}^{k} \tilde{W}^k$ from the proof of \Cref{prop:jacobianConversionToLinear}, we have
    \begin{align*}
        \begin{split}
            \|A^k\|_2 & \leq \|\hat{G}^k\|_2 \| \tilde{W}^k\|_2 = \|G^k\|_2 \| \tilde{W}^k\|_2.
        \end{split}
    \end{align*}
    For $\| \tilde{W}^k \|_2$  we have $(\tilde{W}^{k\top} \tilde{W}^k)_{ij} = \sum_l \tilde{W}^k_{li} \tilde{W}^k_{lj}$.
    With respect to the sparsity pattern of the matrix $\tilde{W}^k$, $\tilde{W}^{k\top} \tilde{W}^k$ is only non-zero on its diagonal with $(\tilde{W}^{k\top} \tilde{W}^k)_{ii} = \| W_{i, :}\|^2_2$.
    Thus $\| \tilde{W}^k \|_2 = \max_i \| W_{i, :}\|_2 = \| W \|_{2 \rightarrow \infty}$. This concludes the proof.
\end{proof}

\paragraph{Proof of \Cref{prop:analyticalSdp}}
\begin{proof}
    Using \Cref{prop:jacobianConversionToLinear} to vectorize the Jacobian of the layer $h^k(x) = \Phi(W^kx)$, we obtain $A^k = \tilde{W}^k$ (see proof of \Cref{prop:jacobianConversionToLinear}). As stated in the proof of \Cref{prop:vectorizedLipschitzIsBetter}, $A^{k\top} A^k$ is diagonal with $(A^{k\top} A^k)_{ii} = \| W_{i, :}\|_2$.
    Next, we state the semidefinite program of \cite{fazlyab2019efficient} for calculating the Lipschitz constant of $dh^k = A^k \phi'(Wx)$. Define
    \[
    M(\rho, T) = \begin{bmatrix}
                -\alpha' \beta' W^{k\top} T W^k - \rho I & \frac{(\alpha' + \beta')}{2}W^{k\top} T\\
                \frac{(\alpha' + \beta')}{2}TW^k & -T + A^{k\top} A^k
            \end{bmatrix},
    \]
    where $T$ is a diagonal non-negative matrix of appropriate dimensions. We have the following optimization problem.
    \begin{align*}
        \begin{split}
            \min_{\rho, T} &\quad \rho\\
            \text{s.t.} &\quad M(\rho, T) \preceq 0.
        \end{split}
    \end{align*}
    As stated in \cite{fazlyab2019efficient}, for any given feasible pair $(\rho, T)$, $\sqrt{\rho}$ is an upper bound on the Lipschitz constant of the desired map.\\
    We first assume that $\alpha' = -\beta'$, implying that the off-diagonal terms of $M(\rho, T) $ will be zero. As a result,  the linear matrix inequality constraint $ M(\rho, T) \preceq 0$ simplifies to satisfying two semidefinite conditions as follows,
    \[
    \begin{cases}
        \beta^{'2} W^{k\top} T W^k \preceq \rho I,\\
        A^{k\top} A^k \preceq T.
    \end{cases}
    \]
    We claim that the optimal solution for this system of constraints is given by
    \[
    \begin{cases}
        T^* = A^{k\top}A^k,\\
        \rho^* = \beta^{'2} \| W^{k\top} T^* W^k \|_2.
    \end{cases}
    \]
    The choice of $\rho^*$ is trivial. 
    Next, it is easy to see that if we instead use $T' = T^* + E$, where $E$ is another non-negative diagonal matrix, we will have $ W^{k\top} T' W^k = W^{k\top} (T^* + E) W^k = W^{k\top}T^*  W^k + W^{k\top} E W^k \succeq W^{k\top} T^* W^k$, where the last inequality follows as $W^{k\top} E W^k$ is a real symmetric positive semidefnite matrix.This concludes the proof of optimality of the proposed solution for the case in which $\alpha' = -\beta'$. 
    
    Next, we consider the scenario in which $|\alpha'| < \beta'$, we observe that a function that is slope-restricted in $[\alpha', \beta']$ is also slope-restricted in $[-\beta', \beta']$. Consequently, the proposed  $\rho^*$ is a feasible point in this case, although it may not be optimal. The case in which $|\beta'| < | \alpha'|$ follows the same argument, yielding a feasible solution $\rho^* = |\alpha'| \|W^k T^* W^k\|_2$.\\
    Thus, we always have $L^p_{dh^k} \leq L_{\phi'} \| \sqrt{T^*}W^k\|_2=\overline{L}^{p, \text{SDP}}_{dh^k}$.
\end{proof}

\begin{corollary}
     Let $p = 2$, and consider the map $h^k(x) = \Phi(W^k x)$ where the derivative of the $i$th activation function  is slope-restricted in $[\alpha_i', \beta_i']$. Then $L^p_{dh^k}(x) \leq \| D \sqrt{T^*} W^k\|_2$, where $D$ is a diagonal matrix with $D_{ii} = \max\{|\alpha'_i|, |\beta'_i|\}$.
\end{corollary}


\section{Calculation of Anchored Lipschitz}
In this section, we elaborate on some aspects of the anchored Lipschitz calculation that we introduced in the main text.

Consider a continuously differentiable function $\phi: \mathbb{R} \mapsto \mathbb{R}$.
The anchored Lipschitz constant of $\phi$ at a $x$ is given by
\begin{align}\label{eqn:anchoredLipschitzSolve}
    L_{\phi}(x) =  \max_{y \neq x} \frac{|\phi(y) - \phi(x)|}{|y - x|}. 
\end{align}





This optimization problem can be solved on a case-by-case basis.
For example, for the case of the $\tanh$ function, the maximizer of  \eqref{eqn:anchoredLipschitzSolve} is the point from which the tangent passes through $(x, \phi(x))$. This is given by solving the nonlinear equation
\begin{align} \label{eqn:nonlinearEquation}
    |\phi'(y)| = \lim_{t \to y}\frac{|\phi(t) - \phi(x)|}{|t - x|}.
\end{align}
See \Cref{fig:anch} for reference. A similar idea follows for other bounded activation functions.
We note that \eqref{eqn:nonlinearEquation} is in general a nonlinear equation without a closed-form solution. In practice, we use a numerical method (like bisection) to solve this nonlinear equation at initiation for a grid of points of the real line and then query these values whenever they are needed for Lipschitz calculation.

Next, consider a single residual block as in \eqref{eq:sequential model}
\begin{align}
    h(x) = H x + G \Phi(W x). \notag
\end{align}
With the definition of anchored Lipschitz, one can use the local naive bound for the Lipschitz constant to obtain the following bound on the Lipschitz constant of $h$,
\begin{align}
    \overline{L}^{p, \text{naive}}_{h}(x) = \|H\|_p + \|G\|_p \|\mathrm{diag}(L_{\Phi}(x)) W\|_p, \notag
\end{align}
where $L_{\Phi}(x) = [L_{\phi_1}(x), \cdots, L_{\phi_{n_1}}(x)]$. 
Then $\overline{L}^{p, \text{naive}}_{h}(x)$ would be an upper bound on the anchored Lipschitz of $h$. This can be adapted to LipSDP~\cite{fazlyab2019efficient} or LipLT \cite{fazlyab2023certified}.

The above analysis can extended to multi-layer residual neural networks.  For example, multiplying the layer-wise anchored Lipschitz bounds will yield the anchored naive Lipschitz bound for the whole network. 


\ifx
\section{Curvature Improved Lipschitz Constants}\label{sec:curvImprovedLipschitz}
The local Lipschitz constant of a function $f$ at a region of space $\mathcal{X}$ is defined as a constant $L_f(\mathcal{X})$ such that
\[
\| f(x) - f(y) \| \leq L_f(\mathcal{X}) \|x - y \|, \quad \forall x, y \in \mathcal{X}.
\]
We propose to update the global Lipschitz constant using the Lipschitz constant of the Jacobian of a function.
\begin{theorem}\label{prop:curvatureImprovedLipschitz}
    Let  $f$ be a differentiable function with $L_{Df}$-Lipschitz Jacobian. Consider the $\varepsilon$-ball around $x$, i.e., $B_\varepsilon(x) = \{ y | \| y - x\| \leq \varepsilon \}$.  Then $\|D f(x)\| + L_{Df}\varepsilon$ is a valid upper bound to the local Lipschitz constant of $f$ over $B_\varepsilon(x)$.
\end{theorem}
\begin{proof}
    Fixing the point $x$, we write the definition of the Lipschitz constant for the Jacobian of the function for any point $y \in B_\varepsilon(x)$
    \begin{align}
        &\|Df(y) - Df(x)\| \leq L_{Df}\|x-y\| \notag \\
        &\rightarrow \|Df(y)\| - \|Df(x)\| \leq L_{Df}\|x-y\|\notag \\
        &\rightarrow L_f(B_\varepsilon(x)) = \sup_{y \in B_\varepsilon(x)} \|Df(y)\| \leq \|D f(x)\| + L_{Df}\varepsilon. \notag
    \end{align}
    Therefore $\|D f(x)\| + L_{Df}\varepsilon$ is a valid upper bound to the Lipschitz constant of $f$ in $B_\varepsilon(x)$.
\end{proof}

Based on \Cref{prop:curvatureImprovedLipschitz}, given a bound on the Lipschitz constant of the Jacobian of a function $f$, we can update its Lipschitz constant as $\min\{L_f, \|Df(x)\| + L_{Df} \varepsilon \}$ at a point $x$.
This Lipschitz constant will be valid for all points in $B_\varepsilon(x)$.
\fi


\section{Time Complexity}
We analyze the time complexity of \Cref{alg:iterativeCurvatureEstimation} when the $\ell_2$ norm is used ($p=2$).
We utilize the power method to calculate the matrix norms. We assume that we perform only a single loop of power iteration to calculate the matrix norms. This assumption is justified in previous work \cite{fazlyab2023certified, huang2021training}.
We provide the complexity in terms of the number of multiplications. For a general neural network of the form of equation \eqref{eq:sequential model} we have the following calculations:
\begin{itemize}
    \item Lipschitz constant of each residual block, i.e. $\overline{L}^p_{h^k} : \mathcal{O}(n_{k+1}n_k + n_{k+1}n'_k + n'_{k}n_k)$.
    \item Lipschitz constant of the Jacobian of each residual block, i.e., $\overline{L}^p_{Dh^k}:  \mathcal{O}(n_{k+1}n'_k + 2n'_{k}n_k)$ or $\overline{L}^p_{dh^k}:  \mathcal{O}(n_{k+1}n'_k + n_{k+1}n'_{k}n_k)$.
    \item Lipschitz constant of the subnetwork from the first layer to the $(k+1)$-th layer, i.e.,
    $$\overline{L}^p_{k+1}: \mathcal{O}(\sum_{j=0}^k \big[ n_{j+1}n_{j} + n'_{j}n_{j} + \sum_{i=j+1}^k (n_{i+1}n_{i} + n_{i+1}n'_{i} + n'_{i}n_{i}) \big] + \sum_{i=0}^k (n_{i+1}n_{i} + n_{i+1}n'_{i} + n'_{i}n_{i})).$$
\end{itemize}
However, it is worth mentioning that the computational complexity of a forward pass through the network is also a sum of quadratic terms, i.e.,
$
\mathcal{O}( \sum_{k=0}^K n_{k+1}n_k + n_{k+1}n'_k + n'_{k}n_k).
$
Thus, the main bottleneck would be the calculation of $\overline{L}^p_{k+1}$ and $\overline{L}^p_{dh^k}$.
We leverage the specialized GPU implementation of LipLT \cite{fazlyab2023certified}, which substantially reduces the time complexity of calculating $\overline{L}^p_k$, $k=0, \cdots K$. 

Furthermore, by using 1-Lipschitz networks as done in some of the experiments, the time complexity of $\overline{L}^p_{k+1}$ and $\overline{L}^p_{h^k}$ would be $\mathcal{O}(1)$.

\ifx
\section{Bound Comparison}
In this section, utilizing the vectorized representation introduced earlier, we provide a theoretical comparison of our method with \cite{singla2020second} for the case of a two-layer scalar-valued neural network.
%
\begin{proposition}\label{prop:compositionalIsBetter}
    Let $p = 2$, and consider $f(x) = c^\top \Phi(Wx)$, where $W \in \mathbb{R}^{m \times n}$ and $\phi$ is twice differentiable, slope restricted in $[\alpha, \beta]$ with its derivative slope restricted in $[\alpha', \beta']$ with $\alpha'=-\beta'$. Then $\overline{L}^{p}_{df}$ provides a better bound on the norm of the hessian ($\| \nabla^2 f(x)\|_2$) than the method of \cite{singla2020second}, i.e., $\overline{{L}}^{p}_{df} \leq L_{\phi'} \| W \|^2 \| c \|_\infty$. 
\end{proposition}

\paragraph{Proof of \Cref{prop:compositionalIsBetter}}
\begin{proof}
        
    For $f(x)=c^\top \Phi(Wx)$, the application of \cite{singla2020second} yields the following expression for the Hessian,
    \begin{align}
        \nabla^2 f(x) &\preceq \sum_{i=1}^{m} (\alpha' \min(c_i,0)+ \beta' \max(c_i,0))W_{i,:}^\top W_{i,:} \notag \\
        & = \sum_{i=1}^{m} \beta' |c_i| W_{i,:}^\top W_{i,:} \\
        &= \beta' W \mathrm{diag}(|c|)W^\top
        \notag
     \end{align}
    Therefore, the maximum eigenvalue of the right-hand side is $\beta' \|W\|_2^2 \|c\|_{\infty}$.
    
    the corresponding matrix $A$ of \Cref{prop:jacobianConversionToLinear} is given as
    $A_{jl} = c_l W_{lj}$. To calculate the norm of this matrix, we have for all unit vectors $x$
    \begin{align*}
        \begin{split}
            x^\top A^\top A x &= \sum_{l, k} (A^\top A)_{lk} x_l x_k\\
            &= \sum_{l, k} \sum_j A_{jl} A_{jk} x_l x_k\\
            &= \sum_{l, k} \sum_j W_{lj} W_{kj} c_l x_l c_k x_k\\
           \xrightarrow{y_i = c_i x_i} &= \sum_{l, k} \sum_j W_{lj} W_{kj} y_l y_k\\
           &= \sum_{l, k} (WW^\top)_{lk} y_l y_k \\
           &= y^\top WW^\top y \leq \| W \|_2^2 \|y\|_2^2.
        \end{split}
    \end{align*}
    Next, we note that $y = \mathrm{diag}(c) x$. Consequently,
    \begin{align*}
        \begin{split}
            \| y \|_2 = \| \mathrm{diag}(c) x \|_2 
            &\leq \sup_{\|x \|_2 \leq 1} \| \mathrm{diag}(c) x \|_2 \\
            &= \| \mathrm{diag}(c)\|_2 = \|c\|_\infty.
        \end{split}
    \end{align*}
    As a result, $x^\top A^\top A x \leq \|W\|_2^2 \|c\|^2_\infty$, and thus, $\overline{L}^{p}_{df}=L_{\phi'}\|A\|_2 \|W\|_2 \leq L_{\phi'} \| W\|_2^2 \|c\|_\infty$.
\end{proof}
Although \Cref{prop:compositionalIsBetter}  establishes the advantage of our method over \cite{singla2020second} for a two-layer neural network, this is not necessarily the case for general multi-layer networks; our bounds are not always better for these networks. However, our method applies to more general architectures and activation functions that Hessian-based methods cannot support.
\fi

\section{Experiments}
In this section, we provide the details of our methods and provide further supporting experiments. 

\subsection{Implementation Details}\label{appendix:hyperparameters}
We used three different architectures in our experiments.
We show convolutional layers in the form $C(c, k, s, p)$,
where $c$ is the number of filters, $k$ is the size of the square kernel, $s$ is the stride length, and $p$ is the
symmetric padding. Fully connected layers are of the form $L(n)$, where $n$ is the number of output
neurons of this layer. 
Furthermore, residual layers of the form \eqref{eq:sequential model} is denoted by an extra character `R', i.e., CR and LR for residual convolutional and fully-connected layers, respectively.
The details of our architectures are as follows:
\begin{itemize}
    \item 6C2F: C(32, 3, 1, 1), C(32, 4, 2, 1), C(64, 3, 1, 1), C(64, 4, 2, 1),
C(64, 3, 1, 1) C(64, 4, 2, 1), L(512), L(10).
    \item 6F: L(1024), L(512), L(256), L(256), L(128), L(10).
    \item Lip-3C1F: CR(15, 3, 1, 1), CR(15, 3, 1, 1), CR(15, 3, 1, 1), LR(1024).
\end{itemize}

For the hyperparameter $\lambda$ used in the loss \eqref{eqn:regularizationOfLipschitz}, we employ a primal-dual approach. That is, after each mini-batch we update 
\[
\lambda^+ = \min\{\lambda + \eta (A_c - \varepsilon), \lambda_\text{min}\},
\]
where $\lambda$ is the current value of the regularizer, $\eta$ is a step size, $A_c$ is a moving average of the training accuracy of the mini-batches, $\varepsilon$ is the minimum train accuracy that we expect from the model, and $\lambda_\text{min}$ is the smallest value for the regularizer.
For training on CIFAR-10, we choose $(\eta, \epsilon, \lambda_\text{min})=(0.05, 0.6, 0.01)$.

The rest of the training details are as follows. We use the modified cross-entropy loss function from  \cite{prach2022almost}, 
\begin{equation}
   \mathcal{L}^{\tau, \nu}_\text{CE}(f(x), y) =  \tau \cdot \text{CE}(\frac{f(x) - \nu \sqrt{2}\: \overline{L} u_y}{\tau}, y),
\end{equation}
where $\tau$ is a temperature constant, $u_y$ is the one-hot encoding of the value of $y$, and $\overline{L}$ is an $\ell_2$ Lipschitz constant of the model. $\nu$ is a zero-one variable based on the architecture and mode of training. For the 6C2F and 6F models, as we want regularization directly through the curvature constant, we set $\nu = 0$. For the Lip-3C1F model, we set $\nu = 1$ as the original \cite{araujo2023unified} work.
We use  $\tau = 0.25$.
Furthermore, we train our models for 1000 epochs with a batch size of 256 with a cosine annealing strategy with an initial learning rate of $10^{-4}$ and a final learning rate $10^{-5}$, and report the average results on two seed in \cref{table:main_results}.

\subsection{Per-sample Improvement}
Expanding on the ``Comparison with other Curvature-based Methods'' experiment in \Cref{para:curvatureComparison}, we provide the per-sample improvements of the certified radii in \Cref{fig:persampleImprovement}, corresponding to \Cref{fig:certified_radius}.
\begin{figure}[t!]
    \centering
    \includegraphics[width=0.70\columnwidth]{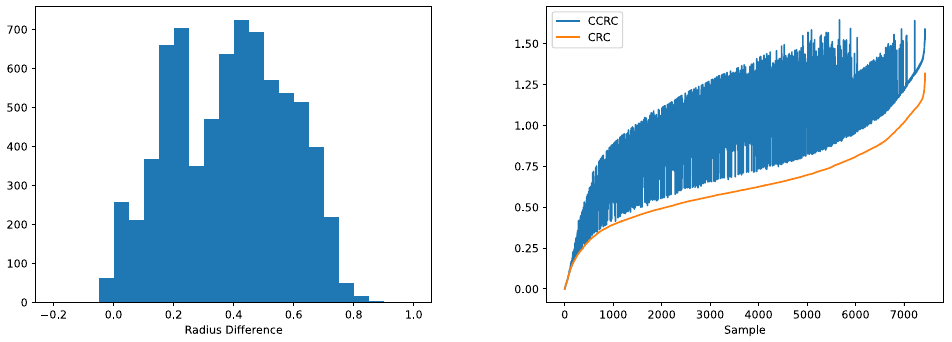}
    \caption{
    Comparison of certified radii acquired via CRC/CCRC on a 6-layer neural network trained via curvature regularization on MNIST. 
    \emph{(Left)} Histogram of per-sample radius improvement of our method over \cite{singla2020second}. 
    \emph{(Right)} Plot of certified radii for correctly classified data.
    The data are sorted according to the  CRC radii.}
    \label{fig:persampleImprovement}
\end{figure}

\subsection{Training with Direct Curvature Regularization}


We observed that for the models that are trained with direct regularization of the curvature, first-order certificates are significantly better than zeroth-order certificates, i.e., by regularizing the model’s curvature, \Cref{prop:compare} would hold for all points of the test dataset. 
This is shown in \Cref{fig:curvatureIsBetterThanLipschitz} for the 6F model.

\begin{figure}[t!]
    \centering
    \includegraphics[width=.80\columnwidth]{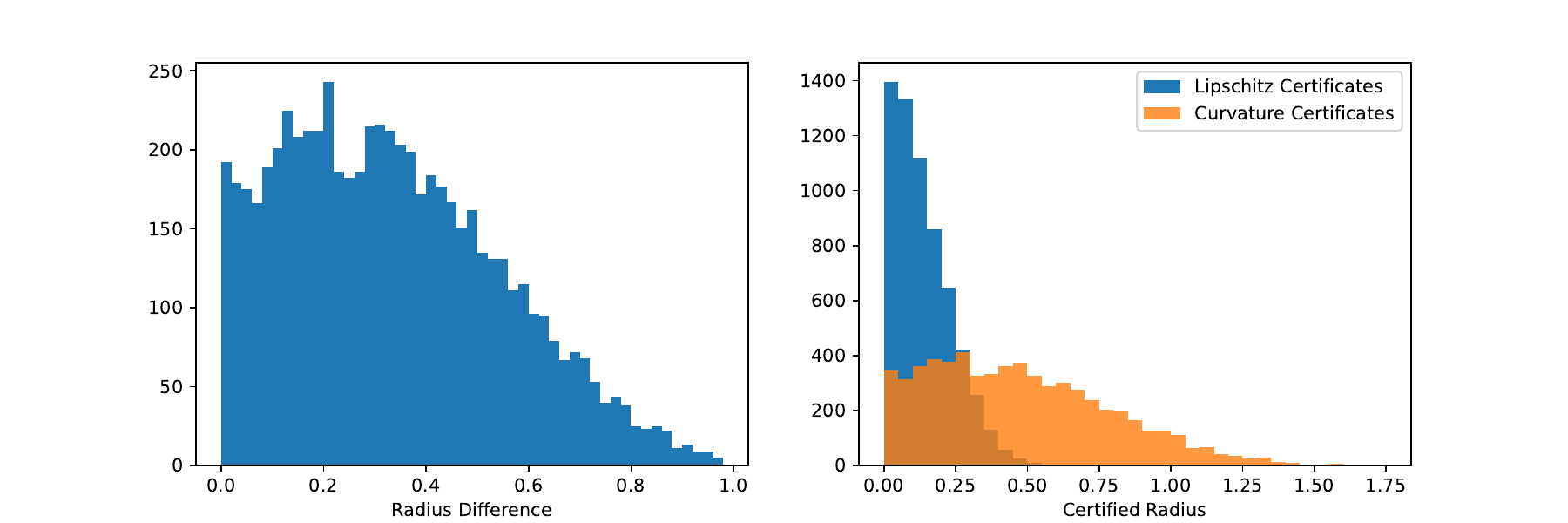}
    \caption{
    Comparison of certified radii calculated through Curvature Certificates versus Lipschitz Certificates for a 6-layer neural network trained via curvature regularization on CIFAR-10.
    \emph{(Left)} Histogram of the per-sample radii improvements.
    \emph{(Right)} Histogram of certified radii.
    }
    \label{fig:curvatureIsBetterThanLipschitz}
\end{figure}

\begin{figure}[t!]
    \centering
    \includegraphics[width=0.4\columnwidth]{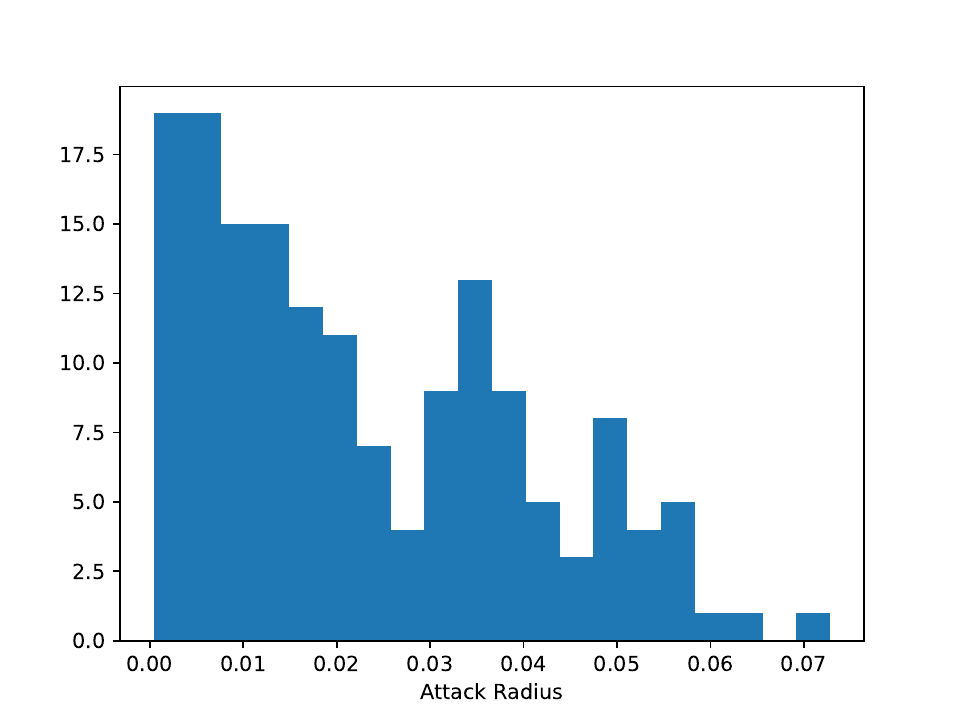}
    \caption{Attack radii certificates for a 1-Lipschitz structure.}
    \label{fig:attackRadius}
\end{figure}

\subsection{Attack Certificates on 1-Lipschitz models}
In this experiment, we provide radii for provable attacks on a 1-Lipschitz model trained on the CIFAR-10 dataset. 
The curvature required for this certificate was calculated using \Cref{alg:iterativeCurvatureEstimation} and utilizing the 1-Lipschitz structure.
\Cref{fig:attackRadius} shows the budget required for a subset of the correctly classified data points that can certifiably be attacked.
We consider the test samples of the CIFAR-10 dataset, which includes 10,000 samples. The model's accuracy on the test set is approximately 50\%, resulting in about 5,000 correctly labeled samples. Of these 5,000 samples, we can provide an attack certificate for approximately 150 of them. This translates to a 3\% success rate (150/5000) for the attack certificate among the correctly classified test samples. It is worth noting that these samples can all be provably misclassified with an attack budget of less than $0.07$, even on 1-Lipschitz networks.

\end{document}